\documentclass[letterpaper, 10 pt, conference]{ieeeconf}  % Comment this line out
\IEEEoverridecommandlockouts                              % This command is only
\let\labelindent\relax

\usepackage{settings/preamble}
\let\barL\L

\newcommand{\eqd}{\mathrel{\overset{\mathrm{d}}{=}}}

\newcommand{\D}{\mathcal{D}}
\renewcommand{\F}{\mathscr{F}}
\renewcommand{\N}{\mathcal{N}}
\renewcommand{\L}{\mathcal{L}}
\renewcommand{\S}{\mathcal{S}}
\renewcommand{\H}{\mathcal{H}}

\newcommand{\X}{\mathcal{X}}

\newcommand{\dTV}{d_{\mathrm{TV}}}
\newcommand{\dKL}{d_{\mathrm{KL}}}
\newcommand{\Vol}{\mathrm{Vol}}

\newcommand{\hL}{\widehat{\mathcal{L}}}
\newcommand{\tL}{\widetilde{\mathcal{L}}}

\newcommand{\eps}{\varepsilon}

\newcommand{\dx}{d_x}
\newcommand{\dxi}{d_{\xi}}
\newcommand{\fs}{f^{\star}}
\newcommand{\xs}{x^{\star}}

\newcommand{\pistar}{\pi^{\star}}

\newcommand{\hq}{\hat{q}}

\newcommand{\hsc}{\hat{s}}

\newcommand{\ssc}{s^{\star}}
\newcommand{\hssc}{\hat{s}^{\star}}
\newcommand{\tssc}{\tilde{s}^{\star}}

\newcommand{\stheta}{\theta^{\star}}
\newcommand{\hstheta}{\hat{\theta}^{\star}}
\newcommand{\tstheta}{\tilde{\theta}^{\star}}

\newcommand{\score}{\mathrm{sc}}

\newcommand{\deltatr}{\delta_{\mathrm{tr}}}

\renewcommand{\vec}{\mathsf{vec}}

\NewDocumentCommand{\hE}{o m}{%
  \IfNoValueTF{#1}{\widehat{\mathbb{E}} \left[ #2 \right]}{\widehat{\mathbb{E}}_{#1} \left[ #2 \right]}%
}

\newcommand{\approptoinn}[2]{\mathrel{\vcenter{%
  \offinterlineskip\halign{\hfil$##$\cr%
    #1\propto\cr\noalign{\kern1pt}#1\sim\cr\noalign{\kern-1pt}}}}%
}
\newcommand{\appropto}{\mathpalette\approptoinn\relax}

\makeatletter
\let\raw@overline\overline
\renewcommand{\overline}[1]{\mathrlap{\raw@overline{\phantom{#1}}}#1}
\let\raw@underline\underline
\renewcommand{\underline}[1]{\mathrlap{\raw@underline{\phantom{#1\!}}}#1}%A negative space is added to avoid overlapping of lines and subscripts.
\makeatother

\usepackage{environ}

\NewEnviron{equationfit}[1][0.9\linewidth]{%
\begin{equation}
  \resizebox{#1}{!}{\(
    \BODY
  \)}
\end{equation}
}
\NewEnviron{equationfit*}[1][0.95\linewidth]{%
\begin{equation*}
  \resizebox{#1}{!}{\(
    \BODY
  \)}
\end{equation*}
}
\NewEnviron{alignfit}[1][0.9\linewidth]{%
\begin{equation*}
  \resizebox{#1}{!}{\(
  \begin{aligned}
    \BODY
  \end{aligned}
  \)}
\end{equation*}
}
\usepackage{cuted}

\title{\LARGE \bf D-IMPL: A Diffusion-based Solver for Parameterized BBOs}

\author{Yang Hu and Na Li% <-this % stops a space
\thanks{This work is supported by NSF CBET 2112085 (AI institute), NSF ECCS 2401390, ONR N000142512173.}% <-this % stops a space
\thanks{Yang Hu and Na Li are both with
        School of Engineering and Applied Sciences,
        Harvard University, Boston, MA
        (email: {\tt\small yanghu@g.harvard.}
        {\tt\small edu} (correspondence), {\tt\small nali@seas.harvard.edu}).}%
}

\begin{document}
  % Math formula settings.
  \setlength{\abovedisplayskip}{3pt}
  \setlength{\abovedisplayshortskip}{2pt}
  \setlength{\belowdisplayskip}{3pt}
  \setlength{\belowdisplayshortskip}{2pt}
  \setlength{\jot}{1pt}  
  \setlength{\floatsep}{1ex}
  \setlength{\textfloatsep}{1ex}
  \allowdisplaybreaks[4]

  \maketitle
  \thispagestyle{empty}
  \pagestyle{empty}

  % Sections
  \begin{abstract}
  Diffusion models have demonstrated strong power in generative modeling tasks across multiple domains, exhibiting a remarkable capability of learning complex distributions from samples. In this paper, we leverage such capability to design an efficient universal diffusion-based solver for parameterized black-box optimizations (BBO), where the optimizer has only \emph{black-box} access to queries of the objective function at the learning stage, yet is able to reduce the additional computational cost at the inference stage for each BBO instance while also capturing the potential \emph{multi-modal} landscape of non-convex objectives. To cast our formulation as a compatible generative modeling task, we introduce the notion of \emph{minimization policy} as a new solution concept, which defines a sampling distribution over the solutions that should concentrate around the minimizer set for each BBO instance. We then propose Diffusion-based Iterative Minimization Policy Learning (D-IMPL), a practical generative-model-based solver for solving parameterized BBOs that employs diffusion models to learn a minimization policy, whose density is proportional to the exponential of the negated objective values, thereby amortizing the computational costs across different BBO instances. Furthermore, we demonstrate the performance of our D-IMPL algorithm by establishing a sample complexity guarantee showing that a $\bm{\delta}$-approximate minimization policy can be effectively learned within $\bm{O(\log(1/\delta))}$ iterations, and by extensive empirical evaluations over a range of constrained and unconstrained BBO tasks.
\end{abstract}
  \section{Introduction}\label{sec:1-introduction}

Diffusion models, especially denoising diffusion probabilistic model (DDPM) and its variants, have recently emerged as a powerful class of generative models for learning complex high-dimensional data distributions \cite{ho2020denoising, lipman2022flow}. By formulating generative modeling tasks as a stochastic denoising process and casting the training process as a score function estimation problem, these models have demonstrated remarkable expressive power and state-of-the-art performance across a wide range of tasks, including text generation \cite{li2022diffusion}, image synthesis \cite{rombach2022high, balaji2022ediff}, audio/video synthesis \cite{ho2022video, bansal2023universal}, molecule generation \cite{jing2022torsional}, and other structured generative modeling problems. A key strength of diffusion models lies in their capability to accurately approximate multi-modal data distributions and generate diversified samples that capture intricate structures in data. This capability, combined with their flexibility in conditional generation, has inspired their applications beyond traditional generative modeling tasks.

Building upon these advances, recent works have begun to explore the use of diffusion models for solving optimization problems \cite{li2024diffusion}. From a generative modeling perspective, optimization can be translated as learning to sample from a distribution that concentrates around the minimizer set. Diffusion models are well-suited for this paradigm due to their ability to represent complex, multi-modal distributions and to generate diverse candidate solutions. Hence this perspective is especially appealing in settings where the landscape is highly non-convex or admits multiple minimizers.

We consider black-box optimization (BBO), a classical category of optimization problems where the explicit analytical form of the objective function is unavailable, so that the optimizer can only access function values via queries at designated inputs. This formulation is well-motivated in practical settings, as real-world objectives are often evaluated through complex experimental or simulation-based processes, making closed-form formulas hard to obtain. Consequently, BBO has been widely studied across a wide range of domains, including power systems, air transportation, material science and drug discovery \cite{alarie2021two}. The primary challenges in BBO arise from the potential non-differentiability and non-convexity of the objective functions, as well as the limited function value queries allowed as the evaluation of the objective function could be expensive and/or time-consuming.

In this work, we introduce Diffusion-based Iterative Minimization Policy Learning (D-IMPL), a DDPM-based framework for solving parameterized BBO families by learning \emph{minimization policies}, a conditional sampling distribution over minimizers. Concretely, for a parameterized objective function $f(x; \xi)$, where the parameter $\xi$ encodes the configuration of the BBO instance, we train a generative model conditioned on $\xi$ such that the resulting minimization policy $\pi(\cdot | \xi)$ has density proportional to $\exp\prn[\big]{ -\varLambda f(x;\xi) }$, where $\varLambda$ is a scaling parameter. Owing to the expressiveness and inherently multi-modal nature of diffusion models, the learned policy is capable of capturing the entire minimizer set without collapsing to any single solution. Furthermore, by simultaneously training on data collected across different parameters, the framework amortizes computational cost and enables generalization to unseen configurations of $\xi$ with minimal additional optimization overhead, and thereby fundamentally differs from both classical instance-wise optimization methods and existing amortized approaches that typically yield a single solution.

The main contributions of this paper are two-fold: (1) we propose a new generative-modeling perspective for solving parameterized BBOs, adopting minimization policies as a new solution concept, and (2) we develop D-IMPL, a learning-based algorithm that leverages DDPM to iteratively update the minimization policy to concentrate around the minimizer set, which exhibits three practically favorable properties:
\begin{itemize}
  \item it assumes only \emph{black-box} access to queries of the objective function values, thereby enabling it to handle non-differentiable and non-convex objectives;
  \item as a \emph{learning-based} approach, once properly trained in the learning stage, it is immediately able to produce solutions for different BBO instances at the inference stage with minimal additional overhead, thereby amortizing computational costs across instances;
  \item it produces a stochastic minimization policy $\pi(\cdot | \xi)$ that, given sufficient data coverage, concentrates around the \emph{entire} minimizer set, which is in stark contrast to classical BBO methods (instance-wise solutions) and the recent amortized-BBO methods (deterministic mapping from $\xi$ to only one of the minimizers), and is therefore better at capturing complex, \emph{multi-modal} landscapes.
\end{itemize}
The performance of D-IMPL is supported both by a theoretical sample complexity guarantee (\Cref{thm:main_complexity-D_IMPL}) and by empirical results across a diverse range of objective function families.

The remaining paper is organized as follows: we first formulate the problem and our new solution concept of minimization policies in \Cref{sec:2-settings}, along with a brief introduction to the DDPM model. Then, in \Cref{sec:3-IMPL}, we build our D-IMPL algorithm in a step-to-step manner, highlighting the key iterative policy update design. Finally, we evaluate the performance of D-IMPL by a theoretical sample complexity analysis (see \Cref{sec:4-analysis}) and simulation results (see \Cref{sec:5-simulation}). Due to page limits, the complete proofs can be found in \cite{hu2026draft} on arXiv.

\subsection{Related Work}

Recent work has begun exploring the use of diffusion models for optimization tasks. For instance, \cite{hoseinpour2025diffopf} proposes a physics-guided diffusion approach for solving the optimal power flow (OPF) problem, which is inherently nonlinear and non-convex. Another line of research investigates diffusion-based model predictive control (MPC), designing methods to handle high-dimensional, constrained, and non-convex optimization problems \cite{zhou2024diffusion, huang2024toward, romer2024diffusion}. Although these approaches demonstrate significant improvements over existing learning-based techniques, they are typically tailored to certain applications rather than addressing generic optimization problems.

The proposed method is also closely related in spirit to the literature on amortized optimization methods \cite{amos2023tutorial}, as both consider parameterized families of objective functions. Nevertheless, whereas existing amortized optimization methods aim to learn a deterministic mapping from parameters to a single optimizer, our formulation generalizes the solution concept to minimization policies that learn a distribution concentrating around the entire minimizer set. As discussed previously, such a formulation fits better with non-convex landscapes where the minimizer set may be multi-modal, continuous, or even forming a lower-dimensional manifold structure.

Notably, the works most closely related to ours are \cite{li2024diffusion} and \cite{krishnamoorthy2023diffusion}. However, both approaches operate in a per-instance manner, requiring retraining of the diffusion-based sampler for each new objective function. Additionally, they focus on offline settings in which only a fixed dataset is available, albeit under slightly different assumptions---in \cite{li2024diffusion}, the data may include noisy evaluations or preference-based comparisons, whereas in \cite{krishnamoorthy2023diffusion}, the dataset consists of exact point-value pairs for a single non-parameterized BBO problem. On the contrary, our work considers parameterized families of BBO problems and seeks to learn a universal optimization policy online that is applicable across the entire family, thereby eliminating the efforts needed for retraining on new instances.
  \section{Settings}\label{sec:2-settings}

In this paper, we consider a generic constrained parameterized BBO setting, where we allow both the objective and the constraint set to be parameterized. Formally, consider the following \emph{parameterized} family of \emph{constrained} optimizations:
\begin{equation*}
  \min_{x \in \X(\xi)}\; f(x; \xi),
\end{equation*}
where $\xi \in \varXi \subset \R^{\dxi}$ is the \emph{parameter}, and $x \in \R^{\dx}$ is the \emph{decision variable}; $f: \R^{\dx} \times \varXi \to \R$ denotes a family of \emph{objective functions} parameterized by $\xi$, which are allowed to be non-convex in $x$; $\X: \varXi \to \mathcal{P}(\R^{\dx})$ denotes a family of constraint sets, also parameterized by $\xi$. We highlight that the above formulation features sufficient flexibility by enabling parameterization of both the objective function $f(\cdot; \xi)$ and the constraint set $\X(\xi)$. %thereby enhancing the general applicability and practical utility of our solutions.

The objective of this paper is to design a \emph{learning-based} algorithm that simultaneously solves for minimizers for the entire parameterized family of optimization problems; in other words, we would like to learn a universal solver in one shot, which, at inference time, outputs the minimizer(s) with respect to any input parameter $\xi \in \varXi$ \emph{at minimum additional computation costs}. The rationale behind our objective will soon be elaborated in the following subsection.

To facilitate the theoretical analysis, we make the following assumptions to ensure a well-behaved problem formulation.

\begin{assumption}[Compact Constraint Set]\label{assum:2-1-minimizers}
  For each $\xi \in \varXi$, the constraint set $\X(\xi)$ is a compact subset of $\R^{\dx}$. Further assume that there exists a positive constant $C_x$, such that $\X(\xi) \subseteq \mathbb{B}(0, C_x)$ for all $\xi \in \varXi$.
\end{assumption}

\begin{assumption}[Lipschitz Objective Function]\label{assum:2-2-Lipschitz}
  For each $\xi \in \varXi$, the objective function $f(\cdot; \xi)$ is $K$-Lipschitz continuous in $x$, i.e., $\norm{f(x'; \xi) - f(x; \xi)} \leq K \norm{x' - x}$, $\forall x, x' \in \R^{\dx}$.
\end{assumption}

\subsection{Minimization Policies: A New Solution Concept}\label{sec:settings-formulation}

To ensure the formulation compatible with diffusion models, and to account for multi-modal solutions, we propose to add stochasticity into the solution concept. In particular, instead of a deterministic solution mapping, we would like to learn a \emph{minimization policy} $\pi: \varXi \to \Delta\prn[\big]{\X(\xi)}$ that ``concentrates well'' around the minimizer set $L_0(\xi)$; formally, for any error level $\delta \in (0,1)$ and any accuracy level $\alpha > 0$, we expect to learn a minimization policy $\pi$ such that
\begin{equation*}
  \Prob[x \sim \pi(\cdot | \xi)]{f(x; \xi) \leq \fs(\xi) + \alpha} > 1 - \delta,\quad
  \forall \xi \in \varXi.
\end{equation*}
We may also write the requirement in equivalent form as
\begin{equation*}
  \pi\prn[\big]{L_{\alpha}(\xi) \big| \xi} := \Prob[x \sim \pi(\cdot | \xi)]{x \in L_{\alpha}(\xi)} > 1-\delta,\quad
  \forall \xi \in \varXi.
\end{equation*}
We highlight that the minimization policy proposed here can be viewed as a natural generalization of the classical uni-modal solution concept, where the latter can be viewed as a point-mass distribution. Moreover, the new solution concept is also compatible with the distributional nature of diffusion models so that the minimization policies can be effectively implemented by them. Therefore, the research question can now be formulated as follows: \textit{can we leverage diffusion models to provably efficiently learn a minimization policy that concentrates well around the minimizer set?}

As briefly discussed in the introduction, by employing diffusion models to learn minimization policies, we could obtain a \emph{universal, learning-based} solver for a parameterized family of constrained optimizations using only \emph{black-box} information. Here by ``universal, learning-based method'' we mean our optimizer is first trained on data at the learning stage, and then, at the inference stage, directly outputs the queried minimizers at \emph{minimal additional overhead}; by ``black-box information'' we mean we only require access to function value queries at the learning stage.

\vspace{-2mm}
\begin{table}[H]
    \centering
    \begin{tabular}{c|c|c|c}
      \specialrule{1.5pt}{0pt}{0pt}
      \textbf{Method} & \textbf{Queries} & \textbf{Stage} & \textbf{Solution} \\
      \specialrule{0.4pt}{0pt}{0pt}
      gradient-based methods & gradient & inference & point \\
      classical BBO methods & value & inference & point \\
      \textbf{diffusion-based method} & \textbf{value} & \textbf{training} & \textbf{sampler} \\
      \specialrule{1.5pt}{0pt}{0pt}
    \end{tabular}
    \caption{A comparison of different optimization methods.}\label{tab:3-function_access_comparison}
\end{table}
\vspace{-4mm}

\subsection{Auxiliary Definitions}

For clarity and conciseness, we introduce the following auxiliary notations. Define the \emph{minimum function} of $f(\cdot; \cdot)$ as
\begin{equation*}
  \fs(\xi) := \min_{x \in \X(\xi)}\; f(x; \xi),\quad \forall \xi \in \varXi,
\end{equation*}
which maps each parameter $\xi$ to the minimum value of $f(\cdot; \xi)$. Further, for $\alpha \geq 0$, define the \emph{$\alpha$-level set} of $f(\cdot; \xi)$ as
\begin{equation*}
  L_{\alpha}(\xi) := \set{x \in \X(\xi) \mid f(x; \xi) \leq \fs(\xi) + \alpha},\quad \forall \xi \in \varXi.
\end{equation*}
In particular, $L_0(\xi)$ denotes the \emph{minimizer set} of $f(\cdot; \xi)$, i.e.,
\begin{equation*}
  L_0(\xi) := \arg\min_{x \in \X(\xi)}\; f(x; \xi),\quad \forall \xi \in \varXi.
\end{equation*}
We point out that $\alpha$-level sets provide a natural characterization of the concentration properties of sampler distributions around the minimizer set for highly non-convex objectives.

\subsection{Denoising Diffusion Probabilistic Models (DDPM)}

Denoising diffusion probabilistic models are a prominent family of diffusion models that are widely-used in practice \cite{ho2020denoising}. On a high level, given a target distribution $z_0 \sim q_0$, DDPM simulates a $K$-step \emph{forward} process to gradually corrupt $z_0$ into a unit Gaussian noise $z_K \sim \N(\bm{0}, I)$, and then designs a corresponding \emph{backward} process to restore $z_0 \sim q_0$ from $z_K$. More specifically, given a noise schedule $\set{\beta_k}_{k=1}^{K}$, the distribution of the forward process trajectories is
\begin{equation*}
  q(z_{1:K} | z_0) = q_0(z_0) \textstyle\prod_{k=1}^{K} q_k(z_k | z_{k-1}),
\end{equation*}
where the corruption kernel $q_k(z_k | z_{k-1})$ is
\begin{equation*}
  q_k(z_k | z_{k-1}) \eqd \N\prn[\big]{ z_k; \sqrt{1-\beta_k} z_{k-1}, \beta_k I }.
\end{equation*}
Since the forward process corrupts data with independent Gaussian noises, by additivity we have the marginal
\begin{equation*}
  q(z_k | z_0) \eqd \N\prn[\big]{ z_k; \sqrt{\bar{\alpha}_k} z_0, (1-\bar{\alpha}_k) I },
\end{equation*}
where we define $\alpha_k := 1 - \beta_k$ and $\bar{\alpha}_k := \prod_{s=0}^{k} \alpha_s$. To reverse the forward process, first note an important fact that the forward process can be viewed as a discretized diffusion SDE, and that SDEs can always be reversed using reverse-time Brownian motions  \cite{anderson1982reverse}. Therefore, the backward process is also Gaussian, so that we may apply the restoration kernel
\begin{equation*}
  p^{\theta}(z_{k-1} | z_k) \eqd \N\prn[\big]{ \mu^{\theta}(z_k, k), \varSigma^{\theta}(z_k, k) }.
\end{equation*}
DDPM uses constant isotropic covariance $\varSigma^{\theta}(z_k, k) \equiv \tilde{\sigma}_k^2 I$, and the following parameterization of the mean:
\begin{equation*}
  \mu^{\theta}(z_k, k) = \frac{1}{\sqrt{\alpha_k}} \prn*{z_k - \frac{\beta_k}{\sqrt{1 - \bar{\alpha}_k}} \omega^{\theta}(z_k, k)}.
\end{equation*}
Here $\omega^{\theta}(\cdot, \cdot)$ is a neural network predicting the noise from observed $z_k$ at the $k$\tsup{th} step, which is trained using the loss
\begin{equation*}
  \L_k^{\mathrm{DDPM}}(\theta) := \E[z_0 \sim q_0, \omega_k \sim \N(\bm{0}, I)]{\norm[\big]{ \omega_k - \omega^{\theta}(z_k, k) }^2}.
\end{equation*}
We point out that the above noise-predicting formulation of DDPM is closely related to score matching. Indeed, since the corruption kernel is Gaussian, the score function $\ssc_k(z) := \nabla_{z} \log q_k(z | z_0)$ can be calculated as
\begin{equation}
  \ssc_k(z) = \frac{-\omega_k}{\sigma_k} = \frac{\sqrt{\bar{\alpha}_k} z_0 - z_k}{\sigma_k^2},
\end{equation}
where $\sigma_k := \sqrt{1 - \bar{\alpha}_k}$ is the covariance of the noise term.

In the following sections, we will stick to the noise schedule $\set{\beta_k}_{k=1}^{K}$ induced by $\bar{\alpha}_k := \exp(-2t_k)$, with constants $0 < t_0 < \cdots < t_K$, the resulting process of which is known as the \emph{discretized Ornstein-Unlenbeck process} \cite{oksendal2003stochastic} that is favored for theoretical analysis.
  \section{Algorithm Design}\label{sec:3-IMPL}

In this section, we present Diffusion-based Iterative Minimization Policy Learning (D-IMPL), our algorithm solving parameterized BBO families that learns a desired minimization policy using DDPM.

\subsection{Minimization Policy Design} 

The key to our algorithm framework is how to define a proper ground-truth minimization policy $\pistar(\cdot | \xi)$. Conceptually, the density of the policy $\pistar(x | \xi)$ should always be non-negative, but also have a steep negative correlation with the function values $f(x; \xi)$ to place most densities around the minimizer set. Therefore, a natural choice would be
\begin{equation}\label{eq:3-1-sampling_policy_ideal}
  \pistar(x | \xi) \propto \pi_0(x | \xi) \exp\prn[\big]{ -\varLambda f(x; \xi) }.
\end{equation}
To ensure good concentration around the minimizer set, $\varLambda$ has to be sufficiently large, as shown in the following lemma.

\begin{lemma}[Concentration of Exact Policy Update]\label{thm:3-1-sampling_policy_ideal}
  Under \Cref{assum:2-1-minimizers} and \ref{assum:2-2-Lipschitz}, let $\pistar(\cdot | \xi)$ be a conditional distribution such that its density $\pistar(x | \xi) \propto \pi_0(x | \xi) \cdot \exp\prn[\big]{ -\varLambda f(x; \xi) }$, where $\pi_0(\cdot | \xi)$ denotes the uniform distribution $\mathsf{Unif}\prn[\big]{ \X(\xi) }$. Then, given any accuracy level $\alpha > 0$ and any error level $\delta \in (0, 1)$, for sufficiently large $\varLambda$ such that
  \begin{equation*}
    \varLambda > \frac{2}{\alpha} \prn*{ \dx \log\prn*{\frac{2}{\alpha}} + \log\prn*{\frac{1}{C_0 \delta}} },
  \end{equation*}
  where the constant $C_0$ is defined in \Cref{thm:2-2-level_set_lemma}, we have
  \begin{equation}\label{eq:3-1-concentration_ideal}
    \pistar\prn[\big]{ L_{\alpha}(\xi) \big| \xi } \geq 1 - \delta,\quad \forall \xi \in \varXi.
  \end{equation}
\end{lemma}

\begin{proof}
  See \Cref{sec:apdx-proof_idea_update} of \cite{hu2026draft}.
\end{proof}

However, when the constant $\varLambda$ is large, it is evident that the exponential density of $\pistar$ is numerically unfavorable in actual implementations. In particular, for diffusion-based implementations of the minimization policy, our theoretical analysis will reveal later that the quality of the learned minimization policy deteriorates when $\varLambda$ is large (see \Cref{thm:main_complexity-D_IMPL} for details). Therefore, despite its theoretical conciseness, it is impractical to directly implement the one-step scheme \eqref{eq:3-1-sampling_policy_ideal}.

\subsection{Iterative Minimization Policy Learning (IMPL)}

A natural modification of the one-step scheme is to resort to iterative updates, resulting in a multi-step scheme. To be precise, the algorithm iteratively learns a sequence $\set{\pi_h}_{h=0}^{H}$ of policies, where $\pi_0(x | \xi)$ shall be selected as $\mathsf{Unif}\prn[\big]{\X(\xi)}$ for better data coverage. Ideally, to recover the ground-truth policy $\pistar_H$ defined in \eqref{eq:3-1-sampling_policy_ideal} with $\varLambda = \lambda H$, we can perform a sequence of iterative updates in the form
\begin{equation*}
  \pi_h(x | \xi) \propto \pi_{h-1}(x | \xi) \cdot \exp\prn[\big]{-\lambda f(x;\xi)},\quad \forall h \in [H],
\end{equation*}
to fully recover $\pi_H = \pistar_H$, which we will refer to as the \emph{exact} iterative updates. We point out that each exact update step from $\pi_{h-1}$ to $\pi_h$ can be viewed as a mirror descent with Kullback-Leibler (KL) proximal divergence term, i.e. solving the following optimization problem:
\begin{equation*}
  \pi_h(x | \xi) = \arg\min_{\pi} f(x; \xi) + \tfrac{1}{\lambda} \dKL\prn[\big]{ \pi(\cdot | \xi) \big\Vert \pi_{h-1}(\cdot | \xi) }.
\end{equation*}
Details of the equivalence can be found in \cite{beck2003mirror}. 

In practical implementations, however, each iterative update step is subject to some approximation error, and the step-wise error will propagate and even be magnified into later updates. Therefore, for practical considerations, we will analyze the following \emph{practical} iterative update rule instead:
\begin{equation}\label{eq:3-2-sampling_policy_iterative}
  \pi_h(x | \xi) \appropto \pi_{h-1}(x | \xi) \cdot \exp\prn[\big]{ -\lambda f(x; \xi) },\quad
  \forall h \in [H].
\end{equation}
Here the symbol ``$\appropto$'' reads ``approximately proportional to''. Precisely speaking, for a sequence of learning error thresholds $\set{\eps_h}_{h=1}^{H}$, define the auxiliary distribution
\begin{equation*}
  \nu_h(x | \xi) := \frac{ \pi_{h-1}(x | \xi) \cdot \exp\prn[\big]{ -\lambda f(x; \xi) } }{ \int_{x} \pi_{h-1}(x | \xi) \cdot \exp\prn[\big]{ -\lambda f(x; \xi) } \diff x }
\end{equation*}
as the outcome of the exact update from $\pi_{h-1}$. Then \eqref{eq:3-2-sampling_policy_iterative} is equivalently to requiring $\pi_h$ to be learned to the extent
\begin{equation*}
  \dTV\prn[\big]{ \pi_h(\cdot | \xi), \nu_h(\cdot | \xi) } < \eps_h,\quad
  \forall \xi \in \varXi.
\end{equation*}
In this way, we obtain the following schematic iterative update algorithm framework for learning the minimization policy.

\begin{algorithm}[t]
    \caption{Iterative Minimization Policy Learning (IMPL)}\label{alg:3-IMPL}
    \begin{algorithmic}[1]
      \Require Learning error scheme $\set{\eps_h}_{h=1}^{H}$.
      \State Initialize the policy $\pi_0(\cdot | \xi) \gets \mathsf{Unif}\prn[\big]{\X(\xi)}$.
      \For{$h=1,2,\ldots,H$}
        \State Update $\pi_h(x | \xi) \appropto \pi_{h-1}(x | \xi) \cdot \exp\prn[\big]{ -\lambda f(x; \xi) }$ using generative models up to an error $\eps_h$.
      \EndFor
    \end{algorithmic}
\end{algorithm}

It is intuitive that, when each $\eps_h$ is sufficiently small, we should be able to come up with a similar concentration result as \eqref{eq:3-1-concentration_ideal}. This is reflected in the following theorem.

\begin{theorem}[Concentration of IMPL]\label{thm:3-2-sampling_policy_iterative}
  Under \Cref{assum:2-1-minimizers} and \ref{assum:2-2-Lipschitz}, suppose we learn a sequence of policies $\set{\pi_h}_{h=1}^{H}$ with learning error thresholds $\set{\eps_h}_{h=1}^{H}$ by \Cref{alg:3-IMPL}. Then for any $\alpha > 0$ and sufficiently large $H \in \Z_+$ such that
  \begin{equation*}
    H > \frac{2}{\lambda \alpha} \prn*{ \dx \log\prn*{\frac{2}{\alpha}} + \log\prn*{\frac{2}{C_0 \delta}} },
  \end{equation*}
  where the universal constant $C_0$ is defined in \Cref{thm:2-2-level_set_lemma}, and for sufficiently small learning error thresholds
  \begin{equation*}
    \eps_h < \frac{C_0 \cdot \alpha^{\dx}}{8H} \exp\prn[\big]{ -\lambda\alpha (H-h) } \delta,\quad \forall h \in [H],
  \end{equation*}
  we have, for any $0 < \delta < \min\brac[\big]{ 1, 4H\prn[\big]{ 1 - \exp(-\lambda\alpha) } }$,
  \begin{equation}
    \pi_H\prn[\big]{ L_{\alpha}(\xi) \big| \xi } > 1 - \delta,\quad
    \forall \xi \in \varXi.
  \end{equation}
\end{theorem}

\begin{proof}
  See \Cref{sec:apdx-proof_iterative_update} of \cite{hu2026draft}.
\end{proof}

\subsection{D-IMPL: IMPL Implemented by DDPM}

We point out that, IMPL is a schematic framework that assumes an oracle to perform the iterative updates up to an error threshold. To obtain a practical implementation of IMPL, we still need to specify how to efficiently implement the oracle. As discussed above, diffusion models have proven effective in learning complex, high-dimensional, multi-modal distributions. Therefore, in this section, we proceed to show how to instantiate IMPL with \emph{denoising diffusion probabilistic model (DDPM)} \cite{ho2020denoising}, a prominent example of diffusion models, to achieve practical performance improvement.

In order to correctly leverage DDPM for learning each $\pi_h$ in the IMPL framework, we need to adapt the DDPM model to settle two fundamental issues---conditional generation and shifted data distribution---as discussed below. For notational clarity, we use the subscript $\cdot_{h, k}$ to denote the DDPM-related quantities for the $k$\tsup{th} diffusion step in the $h$\tsup{th} iteration. Note that we always have $z_{h,0} = x_h \sim \pi_h(\cdot | \xi)$.

\textbf{Conditional DDPM.} It is evident that, in order for DDPM to approximate the minimization policy $\pi_h(\cdot | \xi)$, all the related kernels and score functions should also be conditioned on $\xi$. Suppose the sampled parameter $\xi$ follows a data distribution $\rho$, and the score function is chosen from a parameterized family $\F = \set{s^{\theta} | \theta \in \varTheta}$. Then the score function $s^{\theta}(\cdot; \xi)$ is learned by minimizing the following \emph{reweighted} loss over $\theta$:
\begin{equation}\label{eq:4-2-reweighted_DDPM_conditional}
  \L_{h,k}(\theta) := \E[\xi, \omega, x_h \sim \pi_h(\cdot | \xi)]{Z(\xi) \cdot \ell_{h,k}(\theta)},
\end{equation}
where $\ell_{h,k}(\theta) := \norm[\big]{s^{\theta}(x_h; \xi) - \ssc_{h,k}(x_h; \xi)}^2$, $\ssc_{h,k}(x; \xi) = -\omega / \sigma_k$, and the weights $Z(\xi)$ will be introduced soon.

\textbf{Importance Sampling.} Note that vanilla DDPM requires access to samples directly sampled from the target distribution $z_0 \sim q_0$ to compute the error $\omega_k$ and thus the loss $\L_k^{\mathrm{DDPM}}$. However, in our IMPL framework, we require learning the updated policy $\pi_h$ using data sampled from the policy $\pi_{h-1}$ learned in the previous iteration, given that the densities of $\pi_h$ and $\pi_{h-1}$ are only different by a known factor. Such difference leads to the issue of data distribution shift.

To settle this issue, we apply the \emph{importance sampling} trick to the diffusion learning objective. Specifically, consider a \emph{reweighted} DDPM objective $\L_{h,k}$ in the form of \eqref{eq:4-2-reweighted_DDPM_conditional}, where
\begin{equation}
  Z(\xi) := \int_{x} \pi_{h-1}(x | \xi) \exp\prn[\big]{ -\lambda f(x; \xi) } \diff x
\end{equation}
is the (intractable) normalization factor of the update in the $h$\tsup{th} iteration, such that the following relation holds:
\begin{equation*}
  \pi_h(x_h | \xi) = \frac{\pi_{h-1}(x_h | \xi) \exp\prn[\big]{ -\lambda f(x_h;\xi) } }{ Z(\xi) }.
\end{equation*}
Then we can directly rewrite $\L_{h,k}(\theta)$ as
\begin{align}
  \L_{h,k}(\theta) &= \E[\xi, \omega, x_h \sim \pi_{h-1}(\cdot | \xi)]{Z(\xi) \cdot \frac{\pi_h(x_h | \xi)}{\pi_{h-1}(x_h | \xi)} \cdot \ell_{h,k}} \nonumber\\
  &= \E[\xi, \omega, x_h \sim \pi_{h-1}(\cdot | \xi)]{\exp\prn[\big]{ -\lambda f(x; \xi) } \cdot \ell_{h,k}}, \label{eq:4-2-reweighted_DDPM_rewritten}
\end{align}
which now becomes a tractable objective. We point out that, if the hypothesis family $\F$ is \emph{realizable} (i.e. $\ssc_{h,k} \in \F$), then we will still have $\hssc_{h,k} = \ssc_{h,k}$ after reweighting.

\textbf{The D-IMPL Algorithm.} Now we are ready to present the D-IMPL algorithm implemented by DDPM, where we use the empirical version of \eqref{eq:4-2-reweighted_DDPM_rewritten}. Specifically, given $N_h$ i.i.d. samples $\D_{h-1} = \set{(\xi_i, x_i, \omega_i) \mid i \in [N_h]}$ obeying the data distribution $\xi_i \sim \rho$, $x_i \sim \pi_{h-1}(\cdot | \xi)$ and $\omega_i \sim \N(\bm{0}, I)$, the empirical loss can be written as
\begin{equationfit}\label{eq:4-2-reweighted_DDPM_practical}
  \hL_{h,k}(\theta) := \dfrac{1}{N_h} \sum\limits_{i=1}^{N_h} \exp\prn[\big]{ -\lambda f(x_i; \xi) } \cdot \norm*{s^{\theta}(x_i; \xi_i) - (-\omega_i / \sigma_k)}^2.
\end{equationfit}
The complete D-IMPL algorithm can be found in \Cref{alg:3-D_IMPL}.

  \begin{algorithm}[t]
    \caption{D-IMPL: IMPL Implemented by DDPM}\label{alg:3-D_IMPL}
    \begin{algorithmic}[1]
      \Require Learning error scheme $\set{\eps_h}_{h=1}^{H}$, noise schedule $\set{\bar{\alpha}_k}_{k=0}^{K}$, dataset sizes $\set{N_h}_{h=1}^{H}$, hypothesis class $\F$.
      \State Generate a dataset $\D_1$ of size $N_1$ using the uniform initial policy $x_i \sim \pi_0(\cdot | \xi) := \mathsf{Unif}\prn[\big]{\X(\xi)}$.
      \For{$h=1,2,\ldots,H$}
        \State Learn $\set{\hstheta_{h,k}}_{k=0}^{K}$ by minimizing \eqref{eq:4-2-reweighted_DDPM_practical} on $\D_h$ in $\F$.

        \State Generate a dataset $\D_{h+1}$ of size $N_{h+1}$ using the $K$-step DDPM with score functions $\brac[\big]{\hssc_{h,k} := s^{\hstheta_{h,k}}}_{k=0}^{K}$.
      \EndFor
    \end{algorithmic}
  \end{algorithm}
  \section{Theoretical Analysis}\label{sec:4-analysis}

In this section, we provide a sample complexity analysis of the D-IMPL algorithm. The following two assumptions are made, which are both standard in literature.

\begin{assumption}[Hypothesis Class]\label{assum:4-D-diffusion-hypothesis_class}
  The parameterized hypothesis class $\F = \set{s^{\theta} | \theta \in \varTheta}$ has the following properties:
  \begin{enumerate}
    \item Approximate: $\exists \eps_{\F} > 0$, such that $\min\limits_{\theta \in \varTheta} \L_{h,k}(\theta) < \eps_{\F}$.
    \item Bounded: $\exists C_{\F} > 0$, such that for any candidate function $s \in \F$, we have $\norm{s(x; \xi)} \leq C_{\F}$ holds for all $\xi \in \varXi$.
    \item Lipschitz: $\exists K_{\varTheta}$, such that for any $\theta, \theta' \in \varTheta$, we have $\norm{s^{\theta}(x; \xi) - s^{\theta'}(x; \xi)} \leq K_{\varTheta} \norm{\theta - \theta'}$ holds for all $\xi \in \varXi$.
  \end{enumerate}
\end{assumption}

\begin{assumption}[PL Loss]\label{assum:4-D-diffusion-PL_loss}
  For any $h \in [H]$ and $k \in [K]$, $\L_{h,k}(\theta)$ is differentiable and satisfies the Polyak-{\barL}ojasiewicz (PL) condition, i.e., there exists a constant $\mu > 0$, such that
  \begin{equation*}
    \L_{h,k}(\theta) - \L_{h,k}(\tstheta_{h,k}) \leq \frac{1}{2\mu}\norm{\nabla \L_{h,k}(\theta)}^2,\quad \forall \theta \in \varTheta,
  \end{equation*}
  where $\tstheta_{h,k} := \arg\min\limits_{\theta \in \varTheta} \L_{h,k}(\theta)$.
\end{assumption}

Now we are ready for the sample complexity of D-IMPL.

\begin{theorem}[Sample Complexity of D-IMPL]\label{thm:main_complexity-D_IMPL}
  Under Assumptions \ref{assum:2-1-minimizers}, \ref{assum:2-2-Lipschitz}, \ref{assum:4-D-diffusion-hypothesis_class} and \ref{assum:4-D-diffusion-PL_loss}, consider the D-IMPL algorithm (\Cref{alg:3-D_IMPL}) with hypothesis class $\F$ and $K$-step DDPM models with denoising schedule $\set{\bar{\alpha}_k = \exp(-2 t_k)}_{k=0}^{K}$, where we set $K = O\prn[\big]{ 1 / \eps^2 }$, $t_0 = O(\eps)$ and $t_k = O\prn[\big]{ \log(1/\eps) }$. Then, for any $0 < \delta < \min\brac[\big]{ 1, 4H\prn[\big]{1 - \exp(-\lambda\alpha)} }$ and $0 < \deltatr < 1$, with sufficiently large $H$ such that
  \begin{equation*}
    H > \frac{2}{\lambda \alpha} \prn*{ \dx \log\prn*{\frac{2}{\alpha}} + \log\prn*{\frac{2}{C_0 \delta}} },
  \end{equation*}
  where the constant $C_0$ defined in \Cref{thm:2-2-level_set_lemma}, by choosing the following learning error thresholds
  \begin{equation*}
    \eps_h < \frac{C_0 \cdot \alpha^{\dx}}{8H} \exp\prn[\big]{ -\lambda\alpha (H-h) } \delta,\quad \forall h \in [H]
  \end{equation*}
  and providing datasets of sizes at least
  \begin{equation*}
    N_h > \widetilde{O}\prn*{\frac{\log(1/\deltatr)}{\eps_h^2}},\quad \forall h \in [H],
  \end{equation*}
  we have, with probability at least $1-\deltatr$,
  \begin{equation}
    \pi_H\prn[\big]{ L_{\alpha}(\xi) \big| \xi } > 1 - \delta,\quad \forall \xi \in \varXi.
  \end{equation}
\end{theorem}

\begin{proof}
  See \Cref{sec:apdx-proof_D_IMPL} of \cite{hu2026draft}.
\end{proof}
  \section{Numerical Simulation}\label{sec:5-simulation}

In this section, we demonstrate the performance of the practical D-IMPL algorithm by numerical simulations across a diverse range of objective function families.

\textbf{Capability of Capturing Multi-modal Landscapes.}

In \Cref{fig:5-heatmaps-multi_modal}, we show the minimization policies learned by the D-IMPL algorithm for the above objective functions using heatmaps of their densities. It can be observed that the the learned policies concentrate very well around the ground-truth minimizer set, regardless of its topology. Specifically, the proposed method correctly handles the cases where the number of minimizers changes under different parameters (\texttt{Himm}), where minimizers exist in large numbers (\texttt{trig}), and where the minimizer set form a sub-manifold of the space (\texttt{quad-clp} and \texttt{quad-lin}). These results demonstrate D-IMPL's strong capability to capture multi-modal landscapes.

\begin{figure}[t]
  \centering
  \includegraphics[width=\linewidth]{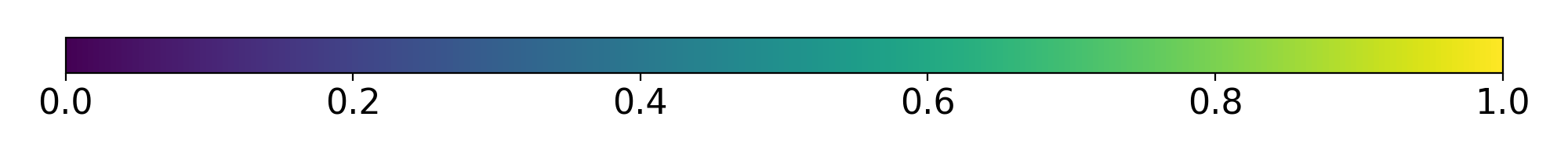}

  \vspace{-2mm}
  \begin{subfigure}[b]{\linewidth}
    \centering
    \includegraphics[width=\linewidth]{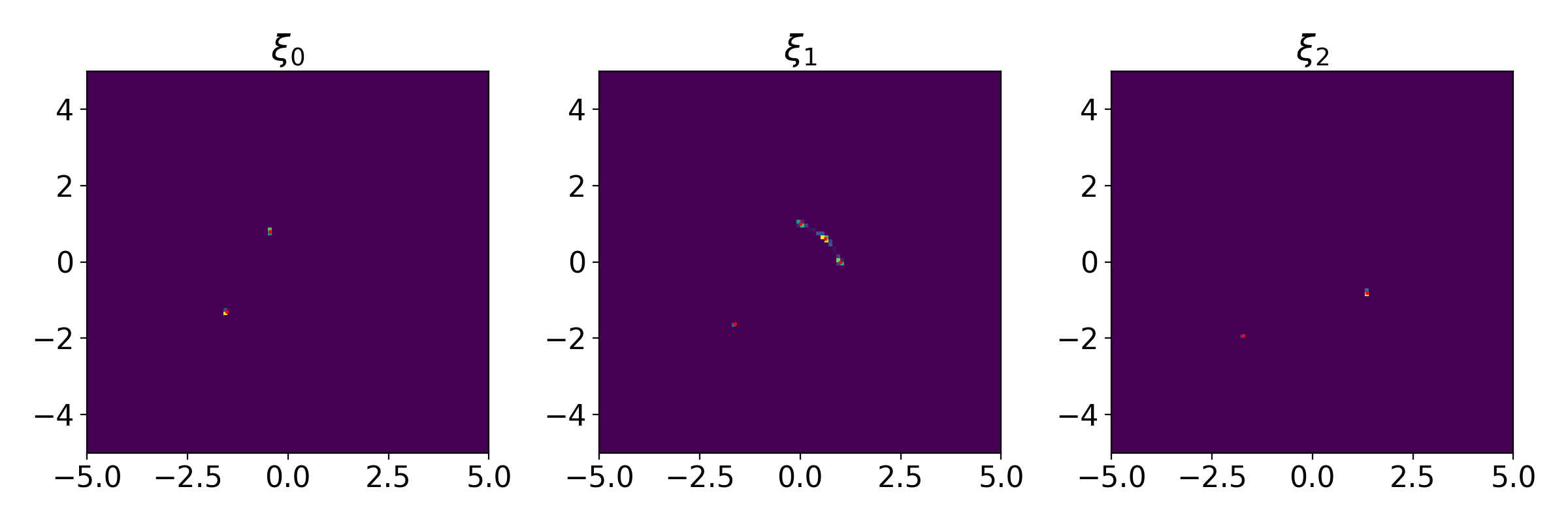}
    
    \vspace{-2mm}
    \caption{\texttt{Himm} (\textcolor{red}{red}: ground truth; \textcolor{Purple}{color}\textcolor{Emerald}{-}\textcolor{Gold1}{scale}: learned)}
  \end{subfigure}

  \begin{subfigure}[b]{\linewidth}
    \centering
    \includegraphics[width=\linewidth]{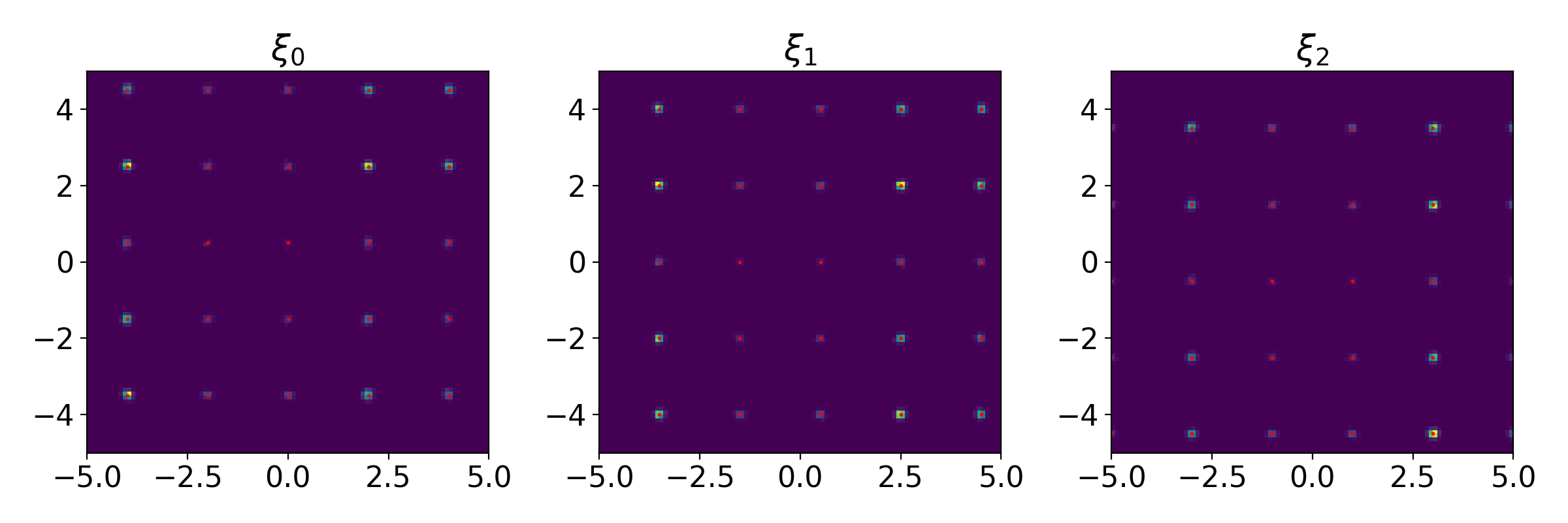}
    
    \vspace{-2mm}
    \caption{\texttt{trig} (\textcolor{red}{red}: ground truth; \textcolor{Purple}{color}\textcolor{Emerald}{-}\textcolor{Gold1}{scale}: learned)}
  \end{subfigure}

  \begin{subfigure}[b]{\linewidth}
    \centering
    \includegraphics[width=\linewidth]{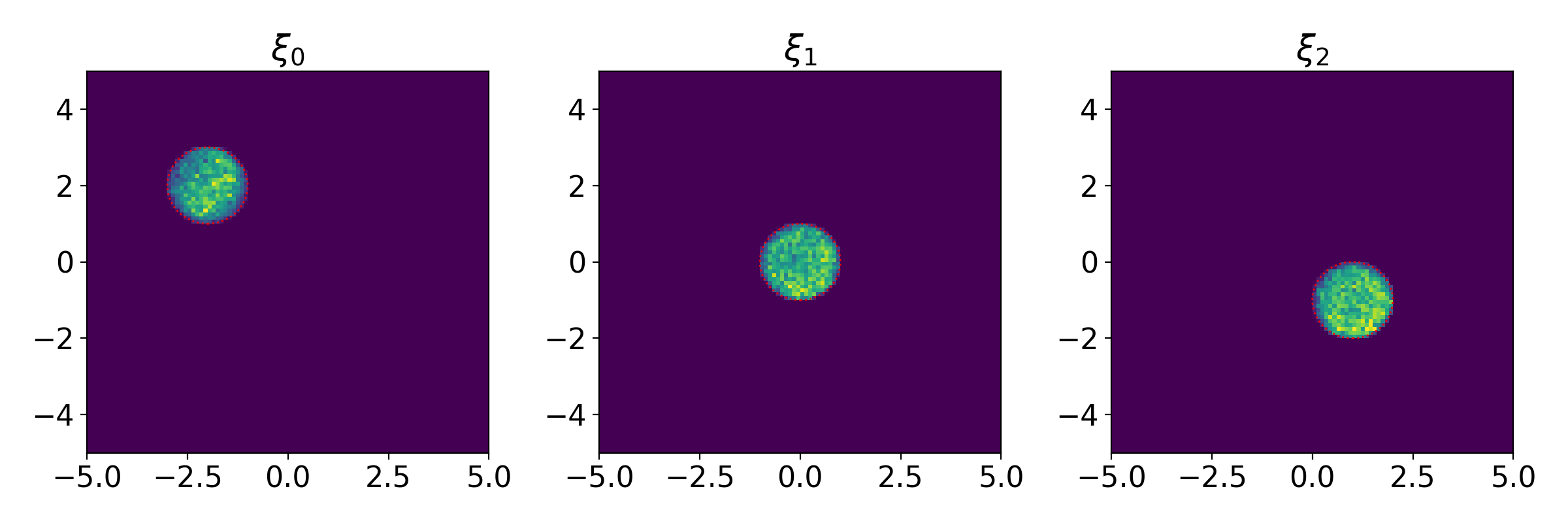}
    
    \vspace{-2mm}
    \caption{\texttt{quad-clp} (dotted disk: ground truth; \textcolor{Purple}{color}\textcolor{Emerald}{-}\textcolor{Gold1}{scale}: learned)}
  \end{subfigure}

  \caption{D-IMPL-learned Policies for Multi-modal Objectives.}\label{fig:5-heatmaps-multi_modal}
\end{figure}

\textbf{Capability of Handling Constraint Sets.} In the following simulations, we demonstrate D-IMPL's capability of handling constraint sets. For this purpose, we consider the following constrained quadratic programming (QP) problem:
\begin{align*}
  \min_{x \in \R^2}~ \tfrac{1}{2} x^{\top} Q x + c^{\top} x,\qquad
  \textrm{s.t.}~Ax \leq b.
\end{align*}
We consider two different parameterized families derived from this problem: (1) fixed constraints (\texttt{qp-fixed}), where $\xi = (\vec(Q), c)$; (2) parameterized constraints (\texttt{qp-param}), where $\xi = (\vec(Q), c, b)$. The minimization policies learned by the D-IMPL algorithm are shown in \Cref{fig:5-heatmaps-constraint_set}. It can be observed that D-IMPL handles both cases quite well, with the learned densities completely inside the constraint set and concentrate closely around the ground truth.

\begin{figure}[t]
  \centering
  
  \begin{subfigure}[b]{\linewidth}
    \includegraphics[width=\linewidth]{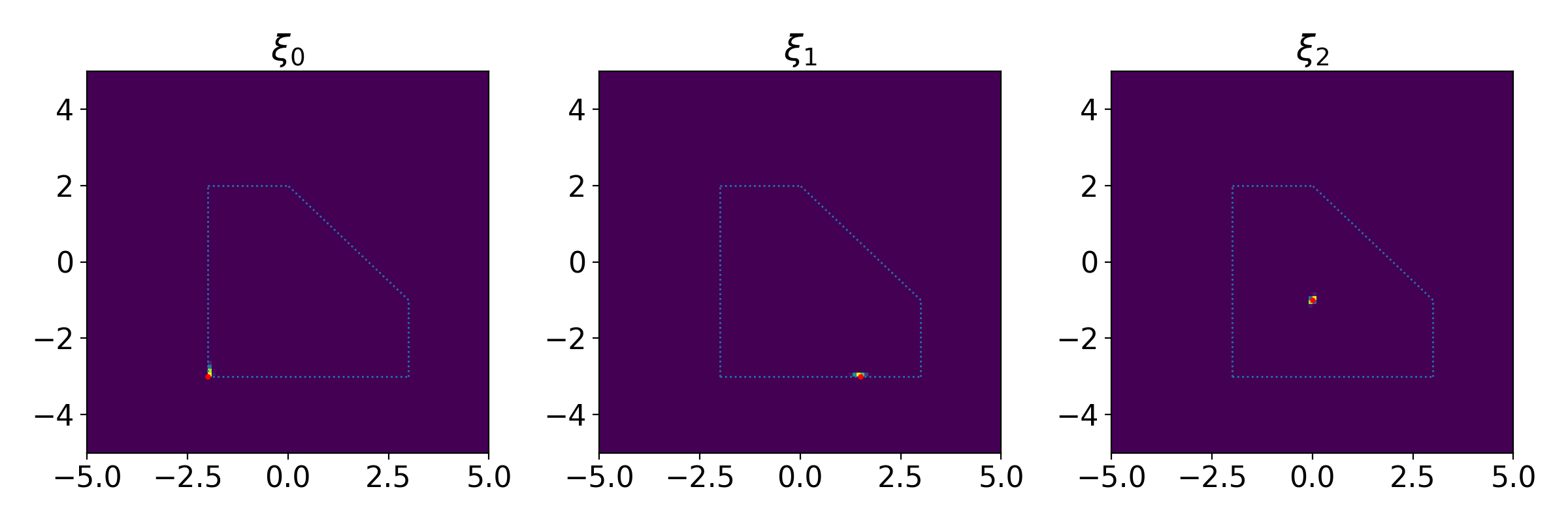}
  
    \vspace{-2mm}
    \caption{\texttt{qp-fixed} (\textcolor{ProcessBlue}{blue-dotted}: fixed constraint set)}
  \end{subfigure}

  \begin{subfigure}[b]{\linewidth}
    \includegraphics[width=\linewidth]{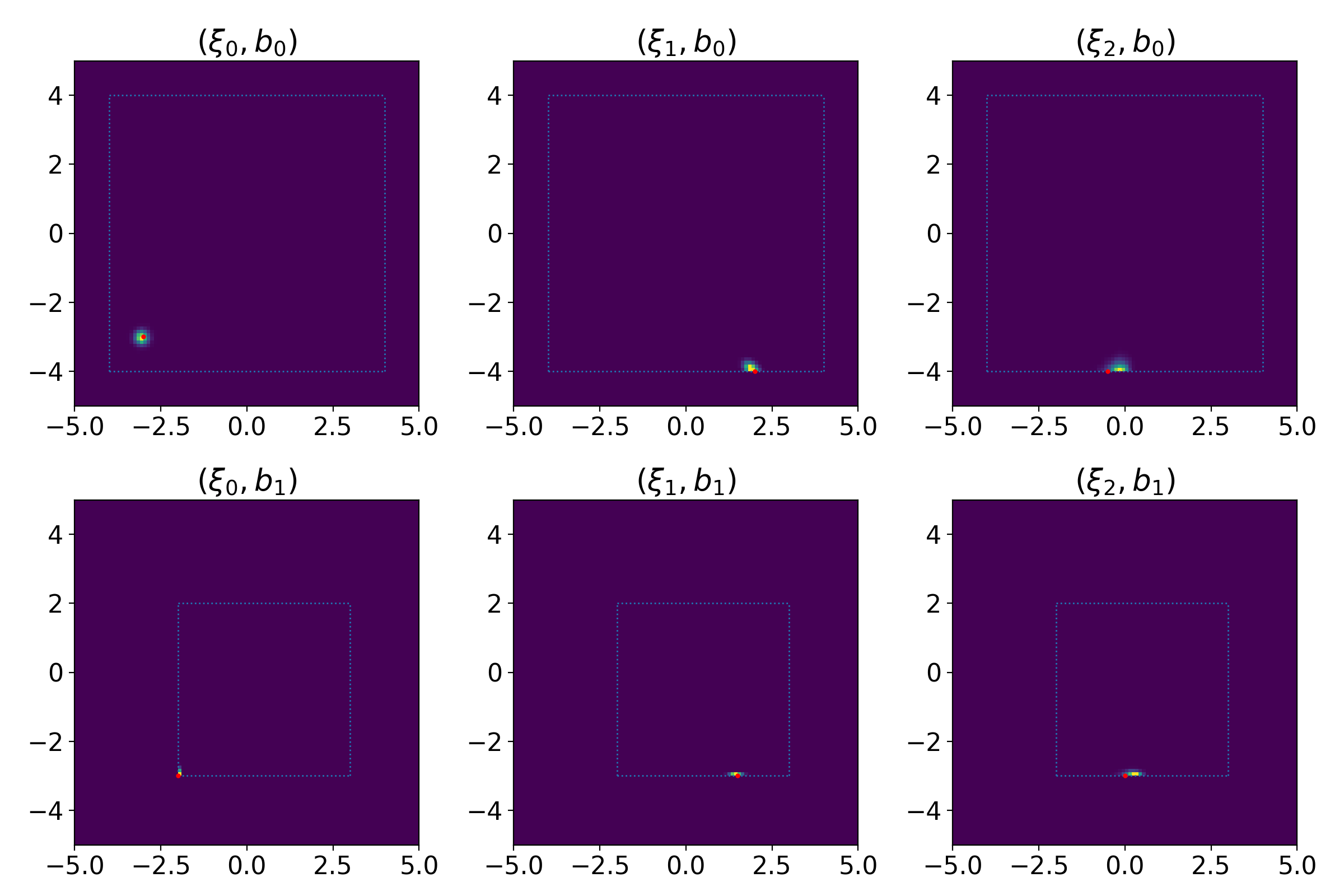}
  
    \vspace{-2mm}
    \caption{\texttt{qp-param} (\textcolor{ProcessBlue}{blue-dotted}: parameterized constraint set)}
  \end{subfigure}

  \vspace{-1mm}
  \caption{D-IMPL-learned Policies for Constrained Problems (\textcolor{red}{red}: ground truth; \textcolor{Purple}{color}\textcolor{Emerald}{-}\textcolor{Gold1}{scale}: learned).}\label{fig:5-heatmaps-constraint_set}
\end{figure}
  
\section{Conclusion}

In this paper, we propose D-IMPL, a DDPM-based algorithm that learns a universal, learning-based stochastic minimization policy for solving parameterized families of constrained BBO problems. The proposed method assumes only black-box access to queries of the objective function, ensures minimal additional computational overhead at the inference stage, and is capable of capturing complex multi-modal landscapes. Future work includes extending the design to other generative models and scaling up the algorithm.

\vspace{-0.5mm}
\section*{Acknowledgment}

We gratefully acknowledge insightful discussions with Yuyang Zhang and Haitong Ma that contributed to this paper.

\vspace{-0.5mm}

  % References
  \bibliographystyle{IEEEtran}
  \bibliography{ref}

@article{hu2026draft,
  title={Solving Black-box Optimization Families using Generative Models},
  author={Hu, Yang and Li, Na},
  journal={arXiv preprint},
  year={2026}
}

@article{gaur2025DDPM,
  title={Improved Sample Complexity For Diffusion Model Training Without Empirical Risk Minimizer Access},
  author={Gaur, Mudit and Trivedi, Prashant and Kunapuli, Sasidhar and Bedi, Amrit Singh and Aggarwal, Vaneet},
  journal={arXiv preprint 2505.18344},
  year={2025}
}

@book{shalev2014understanding,
  title={Understanding machine learning: From theory to algorithms},
  author={Shalev-Shwartz, Shai and Ben-David, Shai},
  year={2014},
  publisher={Cambridge university press}
}

@article{beck2003mirror,
  title={Mirror descent and nonlinear projected subgradient methods for convex optimization},
  author={Beck, Amir and Teboulle, Marc},
  journal={Operations Research Letters},
  volume={31},
  number={3},
  pages={167--175},
  year={2003},
  publisher={Elsevier}
}

@article{alarie2021two,
  title={Two Decades of Blackbox Optimization Applications},
  author={Alarie, St{\'e}phane and Audet, Charles and Gheribi, A{\"\i}men E and Kokkolaras, Michael and Le Digabel, S{\'e}bastien},
  journal={EURO Journal on Computational Optimization},
  volume={9},
  pages={100011},
  year={2021},
  publisher={Elsevier}
}

@article{wang2023recent,
  title={Recent Advances in {B}ayesian Optimization},
  author={Wang, Xilu and Jin, Yaochu and Schmitt, Sebastian and Olhofer, Markus},
  journal={ACM Computing Surveys},
  volume={55},
  number={13s},
  pages={1--36},
  year={2023},
  publisher={ACM New York, NY}
}

@article{ren2025tsrsr,
  title={{TS-RSR}: A Provably Efficient Approach for Batch {B}ayesian Optimization},
  author={Ren, Zhaolin and Li, Na},
  journal={SIAM Journal on Optimization},
  volume={35},
  number={3},
  pages={2155--2181},
  year={2025},
  publisher={SIAM}
}

@inproceedings{chen2022improve,
  title={Improve single-point zeroth-order optimization using high-pass and low-pass filters},
  author={Chen, Xin and Tang, Yujie and Li, Na},
  booktitle={International Conference on Machine Learning},
  pages={3603--3620},
  year={2022},
  organization={PMLR}
}

@inproceedings{ren2023escaping,
  title={Escaping saddle points in zeroth-order optimization: the power of two-point estimators},
  author={Ren, Zhaolin and Tang, Yujie and Li, Na},
  booktitle={International Conference on Machine Learning},
  pages={28914--28975},
  year={2023},
  organization={PMLR}
}

@article{larson2019derivative,
  title={Derivative-free Optimization Methods},
  author={Larson, Jeffrey and Menickelly, Matt and Wild, Stefan M},
  journal={Acta Numerica},
  volume={28},
  pages={287--404},
  year={2019},
  publisher={Cambridge University Press}
}

@article{amos2023tutorial,
  title={Tutorial on Amortized Optimization},
  author={Amos, Brandon and others},
  journal={Foundations and Trends{\textregistered} in Machine Learning},
  volume={16},
  number={5},
  pages={592--732},
  year={2023},
  publisher={Now Publishers, Inc.}
}

@article{ho2020denoising,
  title={Denoising Diffusion Probabilistic Models},
  author={Ho, Jonathan and Jain, Ajay and Abbeel, Pieter},
  journal={Advances in Neural Information Processing Systems},
  volume={33},
  pages={6840--6851},
  year={2020}
}

@article{lipman2022flow,
  title={Flow Matching for Generative Modeling},
  author={Lipman, Yaron and Chen, Ricky TQ and Ben-Hamu, Heli and Nickel, Maximilian and Le, Matt},
  journal={arXiv preprint 2210.02747},
  year={2022}
}

@inproceedings{rombach2022high,
  title={High-resolution Image Synthesis with Latent Diffusion Models},
  author={Rombach, Robin and Blattmann, Andreas and Lorenz, Dominik and Esser, Patrick and Ommer, Bj{\"o}rn},
  booktitle={Proceedings of the IEEE/CVF Conference on Computer Vision and Pattern Recognition},
  pages={10684--10695},
  year={2022}
}

@article{balaji2022ediff,
  title={{eDiff-I}: Text-to-image Diffusion Models with an Ensemble of Expert Denoisers},
  author={Balaji, Yogesh and Nah, Seungjun and Huang, Xun and Vahdat, Arash and Song, Jiaming and Zhang, Qinsheng and Kreis, Karsten and Aittala, Miika and Aila, Timo and Laine, Samuli and others},
  journal={arXiv preprint 2211.01324},
  year={2022}
}

@article{li2022diffusion,
  title={Diffusion-{LM} improves controllable text generation},
  author={Li, Xiang and Thickstun, John and Gulrajani, Ishaan and Liang, Percy S and Hashimoto, Tatsunori B},
  journal={Advances in Neural Information Processing Systems},
  volume={35},
  pages={4328--4343},
  year={2022}
}

@inproceedings{bansal2023universal,
  title={Universal guidance for diffusion models},
  author={Bansal, Arpit and Chu, Hong-Min and Schwarzschild, Avi and Sengupta, Soumyadip and Goldblum, Micah and Geiping, Jonas and Goldstein, Tom},
  booktitle={Proceedings of the IEEE/CVF Conference on Computer Vision and Pattern Recognition},
  pages={843--852},
  year={2023}
}

@article{ho2022video,
  title={Video Diffusion Models},
  author={Ho, Jonathan and Salimans, Tim and Gritsenko, Alexey and Chan, William and Norouzi, Mohammad and Fleet, David J},
  journal={Advances in Neural Information Processing Systems},
  volume={35},
  pages={8633--8646},
  year={2022}
}

@article{jing2022torsional,
  title={Torsional diffusion for molecular conformer generation},
  author={Jing, Bowen and Corso, Gabriele and Chang, Jeffrey and Barzilay, Regina and Jaakkola, Tommi},
  journal={Advances in Neural Information Processing Systems},
  volume={35},
  pages={24240--24253},
  year={2022}
}

@book{oksendal2003stochastic,
  title={Stochastic differential equations},
  author={{\slashO}ksendal, Bernt},
  pages={38--50},
  year={2003},
  publisher={Springer}
}

@article{anderson1982reverse,
  title={Reverse-time diffusion equation models},
  author={Anderson, Brian DO},
  journal={Stochastic Processes and their Applications},
  volume={12},
  number={3},
  pages={313--326},
  year={1982},
  publisher={Elsevier}
}

@inproceedings{krishnamoorthy2023diffusion,
  title={Diffusion Models for Black-box Optimization},
  author={Krishnamoorthy, Siddarth and Mashkaria, Satvik Mehul and Grover, Aditya},
  booktitle={International Conference on Machine Learning},
  pages={17842--17857},
  year={2023},
  organization={PMLR}
}

@article{li2024diffusion,
  title={Diffusion model for data-driven black-box optimization},
  author={Li, Zihao and Yuan, Hui and Huang, Kaixuan and Ni, Chengzhuo and Ye, Yinyu and Chen, Minshuo and Wang, Mengdi},
  journal={arXiv preprint 2403.13219},
  year={2024}
}

@article{hoseinpour2025diffopf,
  title={{DiffOPF}: Diffusion Solver for Optimal Power Flow},
  author={Hoseinpour, Milad and Dvorkin, Vladimir},
  journal={arXiv preprint 2510.14075},
  year={2025}
}

@article{zhou2024diffusion,
  title={Diffusion model predictive control},
  author={Zhou, Guangyao and Swaminathan, Sivaramakrishnan and Raju, Rajkumar Vasudeva and Guntupalli, J Swaroop and Lehrach, Wolfgang and Ortiz, Joseph and Dedieu, Antoine and L{\'a}zaro-Gredilla, Miguel and Murphy, Kevin},
  journal={arXiv preprint 2410.05364},
  year={2024}
}

@article{huang2024toward,
  title={Toward Near-Globally Optimal Nonlinear Model Predictive Control via Diffusion Models},
  author={Huang, Tzu-Yuan and Lederer, Armin and Hoischen, Nicolas and Br{\"u}digam, Jan and Xiao, Xuehua and Sosnowski, Stefan and Hirche, Sandra},
  journal={arXiv preprint 2412.08278},
  year={2024}
}

@article{romer2024diffusion,
  title={Diffusion predictive control with constraints},
  author={R{\"o}mer, Ralf and von Rohr, Alexander and Schoellig, Angela P},
  journal={arXiv preprint 2412.09342},
  year={2024}
}

  % Appendix
  \newpage
  \appendix
  \subsection{Notations}

\vspace{-2mm}
\begin{table}[H]
    \centering
    \begin{tabular}{c|l}
      \specialrule{1.5pt}{0pt}{0pt}
      \textbf{Notation} & \textbf{Explanation} \\
      \specialrule{0.4pt}{0pt}{0pt}
      $\mathcal{P}(S)$ & power set of set $S$ \\
      $\Delta(S)$ & set of all distributions over set $S$ \\
      $\norm{\cdot}$ & Euclidean $\ell_2$ norm of vectors/matrices \\
      $\mathbb{B}(x, r)$ & ball of radius $r$ centered at $x$, i.e., $\set{y \mid \norm{y-x} \leq r}$ \\
      $\appropto$ & ``approximately proportional to'' (definition on page 5) \\
      $\dKL(p \Vert q)$ & Kullback–Leibler divergence of distribution $p$ against $q$  \\
      $\dTV(p, q)$ & total variation distance between distribution $p$ and $q$ \\
      $O(\cdot)$, $\widetilde{O}(\cdot)$ & standard Landau notations for asymptotics \\
      $X \eqd Y$ & random variables $X$ and $Y$ follow the same distribution \\
      $\vec(M)$ & vectorization of matrix $M$ \\
      \specialrule{1.5pt}{0pt}{0pt}
    \end{tabular}
    \caption{Table of notations.}
\end{table}
\vspace{-4mm}

\subsection{A Technical Lemma Regarding Level Sets}\label{sec:apdx-proof_level_set_lemma}

\begin{lemma}\label{thm:2-2-level_set_lemma}
  Under \Cref{assum:2-1-minimizers} and \ref{assum:2-2-Lipschitz}, with $\pi_0(\cdot | \xi)$ denoting the uniform distribution $\mathsf{Unif}\prn[\big]{ \X(\xi) }$, we have
  \begin{equation*}
    \pi_0\prn[\big]{ L_{\alpha}(\xi) \big| \xi }
    \geq C_0 \cdot \alpha^{\dx},
  \end{equation*}
  where constant $C_0 := \pi^{\dx/2} / \prn[\big]{ \Gamma(\dx/2+1) \Vol\prn[\big]{ \X(\xi) } K^{\dx} }$ is uniform for all parameters $\xi \in \varXi$.
\end{lemma}

\begin{proof}
  By \Cref{assum:2-2-Lipschitz}, given any $\xs \in L_0(\xi)$, we have
  \begin{equation*}
    \norm*{f(x; \xi) - f(\xs; \xi)} \leq K \norm{x - \xs} \leq \alpha,~ \forall x \in \mathbb{B}\prn*{ \xs, \tfrac{\alpha}{K} }.
  \end{equation*} 
  Therefore, we have $\mathbb{B}\prn*{ \xs, \tfrac{\alpha}{K} } \subset L_{\alpha}$, and thus
  \begin{equation*}
    \pi_0\prn[\big]{ L_{\alpha}(\xi) \big| \xi }
    \geq \pi_0\prn[\big]{ \mathbb{B}\prn[\big]{ x^*, \tfrac{\alpha}{K} } \big| \xi }
    = \frac{\pi^{\dx/2} \prn[\big]{\frac{\alpha}{K}}^d}{\Gamma \prn[\big]{ \tfrac{\dx}{2}+1 } \Vol\prn[\big]{ \X(\xi) }},
  \end{equation*}
  which completes the proof.
\end{proof}

\subsection{Proof of \texorpdfstring{\Cref{thm:3-1-sampling_policy_ideal}}{Lemma 2}: Concentration of the Ideal Update}\label{sec:apdx-proof_idea_update}

The proof consists of two main steps:
\begin{itemize}
  \item \textbf{Step 1}: Expand $\pistar\prn[\big]{ L_{\alpha}(\xi) \big| \xi }$ in terms of $\pi_0(\cdot | \xi)$ and $\exp\prn[\big]{ - \varLambda f(x; \xi) }$, and lower bound them using the definition of $L_{\alpha}(\xi)$. Intuitively, when $\varLambda$ is sufficiently large, the density of $\pistar(\cdot | \xi)$ should concentrate around $L_{\alpha}(\xi)$.

  \item \textbf{Step 2}: Lower bound $\pi_0\prn[\big]{ L_{\alpha/2}(\xi) \big| \xi }$ using \Cref{thm:2-2-level_set_lemma}.
\end{itemize}
Following the idea above, we have the following proof.

\begin{proof}[Proof of \Cref{thm:3-1-sampling_policy_ideal}]
  Fix a parameter $\xi \in \varXi$, and we proceed to lower bound $\pistar\prn[\big]{ L_{\alpha}(\xi) \big| \xi }$. Indeed, by definition we have
  \begin{align}
    &\pistar\prn[\big]{ L_{\alpha}(\xi) \big| \xi } 
    ={} \int_{L_{\alpha}(\xi)} \pistar(x | \xi) \diff x \nonumber\\
    ={}& \frac{\int_{L_{\alpha}(\xi)} \pi_0(x | \xi) \exp\prn[\big]{ -\varLambda f(x; \xi) } \diff x}{\int_{\X(\xi)} \pi_0(x | \xi) \exp\prn[\big]{ -\varLambda f(x; \xi) } \diff x} \nonumber\\
%    ={}& 1 - \frac{\int_{\X(\xi) \setminus L_{\alpha}(\xi)} \pi_0(x | \xi) \exp\prn[\big]{ -\varLambda f(x; \xi) } \diff x}{\int_{L_{\alpha}(\xi)} \pi_0(x | \xi) \exp\prn[\big]{ -\varLambda f(x; \xi) } \diff x + \int_{\X(\xi) \setminus L_{\alpha}(\xi)} \pi_0(x | \xi) \exp\prn[\big]{ -\varLambda f(x; \xi) } \diff x} \nonumber\\
    \geq{}& 1 - \frac{\int_{\X(\xi) \setminus L_{\alpha}(\xi)} \pi_0(x | \xi) \exp\prn[\big]{ -\varLambda f(x; \xi) } \diff x}{\int_{L_{\alpha}(\xi)} \pi_0(x | \xi) \exp\prn[\big]{ -\varLambda f(x; \xi) } \diff x} \nonumber\\
    \geq{}& 1 - \frac{\int_{\X(\xi) \setminus L_{\alpha}(\xi)} \pi_0(x | \xi) \exp\prn[\big]{ -\varLambda f(x; \xi) } \diff x}{\int_{L_{\alpha/2}(\xi)} \pi_0(x | \xi) \exp\prn[\big]{ -\varLambda f(x; \xi) } \diff x} \nonumber\\
    \geq{}& 1 - \frac{\exp\prn[\big]{ -\varLambda (\fs(\xi) + \alpha) } \pi_0\prn[\big]{ \X(\xi) \setminus L_{\alpha}(\xi) \big| \xi }}{\exp\prn[\big]{ -\varLambda (\fs(\xi) + \alpha/2) } \pi_0\prn[\big]{ L_{\alpha/2}(\xi) \big| \xi }} \nonumber\\
    \geq{}& 1 - \frac{1}{\pi_0\prn[\big]{ L_{\alpha/2}(\xi) \big| \xi }} \cdot \exp(- \alpha \varLambda / 2), \label{eq:3-1-sampling_ideal-pi_bound}
  \end{align}
  where the second inequality is established by the fact that $L_{\alpha / 2}(\xi) \subseteq L_{\alpha}(\xi)$; the third inequality is due to the definition of $\alpha$-level sets such that
  \begin{equation*}
    f(x;\xi) \leq \fs(\xi) + \alpha,~ \forall x \in L_{\alpha}(\xi);
  \end{equation*}
  and the fourth inequality is because $\pi_0\prn[\big]{ \X(\xi) \setminus L_{\alpha}(\xi) \big| \xi } \leq 1$. Note that, by \Cref{thm:2-2-level_set_lemma}, we have a lower bound on the size of the level sets under \Cref{assum:2-1-minimizers} and \ref{assum:2-2-Lipschitz}, i.e.,
  \begin{equation*}
    \pi_0\prn[\big]{ L_{\alpha/2}(\xi) \big| \xi }
    \geq C_0 \cdot \prn*{\frac{\alpha}{2}}^{\dx}.
  \end{equation*}
  Plugging the lower bound back into \eqref{eq:3-1-sampling_ideal-pi_bound}, we have
  \begin{equation}
    \pistar\prn[\big]{ L_{\alpha}(\xi) \big| \xi }
    \geq 1 - \frac{1}{C_0} \prn*{\frac{2}{\alpha}}^{\dx} \cdot \exp(-\alpha \varLambda / 2).
  \end{equation}
  Therefore, if we select $\varLambda$ such that
  \begin{equation*}
    \frac{1}{C_0} \prn*{\frac{2}{\alpha}}^{\dx} \cdot \exp(-\alpha \varLambda / 2) < \delta,
  \end{equation*}
  or equivalently,
  \begin{equation*}
    \varLambda > \frac{2}{\alpha} \prn*{ \dx \log\prn*{\frac{2}{\alpha}} + \log\prn*{\frac{1}{C_0 \delta}} } = O\prn*{\frac{\dx}{\alpha} \log\prn*{ \frac{1}{\alpha \delta}}},
  \end{equation*}
  we shall then guarantee that
  \begin{equation*}
    \pistar\prn[\big]{ L_{\alpha}(\xi) \big| \xi } > 1-\delta.
  \end{equation*}
  The proof is done since the above holds for all $\xi \in \varXi$.
\end{proof}

\subsection{Proof of \texorpdfstring{\Cref{thm:3-2-sampling_policy_iterative}}{Theorem 3}: Concentration of the Iterative Update}\label{sec:apdx-proof_iterative_update}

The proof consists of three main steps:
\begin{itemize}
  \item \textbf{Step 1}: Define auxiliary distributions $\rho_1 = \pistar_H, \ldots, \rho_H = \pi_H$, such that they represent the distributions with approximation errors injected at different steps, and bound the difference $\abs*{ \rho_h\prn[\big]{ L_{\alpha}(\xi) \big| \xi} - \rho_{h-1}\prn[\big]{ L_{\alpha}(\xi) \big| \xi} }$ by $\dTV\prn[\big]{ \pi_h(\cdot | \xi), \nu_h(\cdot | \xi) }$ to accumulate the step-wise approximation errors.

  \item \textbf{Step 2}: Relate $\nu_h\prn[\big]{ L_{\alpha}(\xi) \big| \xi}$ to $\pi_0\prn[\big]{ L_{\alpha}(\xi) \big| \xi}$ to show a lower bound on $\nu_h\prn[\big]{ L_{\alpha}(\xi) \big| \xi}$.

  \item \textbf{Step 3}: Lower bound $\pi_0\prn[\big]{ L_{\alpha/2}(\xi) \big| \xi }$ using \Cref{thm:2-2-level_set_lemma}.
\end{itemize}
To facilitate the analysis, we define a sequence of auxiliary measures as follows:
  \begin{align*}
    \rho_0(x | \xi) &\propto \pi_0(x | \xi) \exp\prn[\big]{ -\lambda H f(x; \xi) }, \\
    \rho_1(x | \xi) &\propto \pi_1(x | \xi) \exp\prn[\big]{ -\lambda (H-1) f(x; \xi) }, \\
    \vdots \\
    \rho_H(x | \xi) &\propto \pi_H(x | \xi).
  \end{align*}
Now, in Step 1, we show a lemma regarding the properties of auxiliary distributions $\rho_h$ and $\rho_{h-1}$.

\begin{lemma}\label{thm:3-2-proof_lemmas-stepwise_error}
  Under the assumptions of \Cref{thm:3-2-sampling_policy_iterative}, we have
  \begin{equation*}
    \abs*{ \rho_h\prn[\big]{ L_{\alpha}(\xi) \big| \xi} - \rho_{h-1}\prn[\big]{ L_{\alpha}(\xi) \big| \xi} }
    \leq \frac{2 \exp(\lambda\alpha (H-h))}{\nu_h\prn[\big]{ L_{\alpha}(\xi) \big| \xi}} \cdot \eps_h
  \end{equation*}
  for any $\xi \in \varXi$.
\end{lemma}

\begin{proof}
  Recall that we define the distribution $\nu_h$ such that
  \begin{equation*}
    \nu_h(x | \xi) \propto \pi_{h-1}(x | \xi) \exp\prn[\big]{ -\lambda f(x; \xi) },
  \end{equation*}
  and the definition of the learning error $\eps_h$ gives
  \begin{equation*}
    \dTV\prn[\big]{ \pi_h(\cdot | \xi), \nu_h(\cdot | \xi) } < \eps_h,\quad
    \forall \xi \in \varXi.
  \end{equation*}
   
  \newpage
  Then, for any $h = 1,\ldots,H$, we have
  
  \begin{strip}
  \vspace{-6mm}
  \begin{align*}
    &\abs*{ \rho_h\prn[\big]{ L_{\alpha}(\xi) \big| \xi} - \rho_{h-1}\prn[\big]{ L_{\alpha}(\xi) \big| \xi} } \\
    ={}& \abs*{\frac{\int_{L_{\alpha}(\xi)} \pi_h(x | \xi) \exp\prn[\big]{-\lambda (H-h) f(x; \xi)} \diff x}{\int_{\X(\xi)} \pi_h(x | \xi) \exp\prn[\big]{-\lambda (H-h) f(x; \xi)} \diff x} - \frac{\int_{L_{\alpha}(\xi)} \pi_{h-1}(x | \xi) \exp\prn[\big]{ -\lambda (H-h+1) f(x; \xi) } \diff x}{\int_{\X(\xi)} \pi_{h-1}(x | \xi) \exp\prn[\big]{ -\lambda (H-h+1) f(x; \xi) } \diff x}} \\
    ={}& \abs*{\frac{\int_{L_{\alpha}(\xi)} \pi_h(x | \xi) \exp\prn[\big]{-\lambda (H-h) f(x; \xi)} \diff x}{\int_{\X(\xi)} \pi_h(x | \xi) \exp\prn[\big]{-\lambda (H-h) f(x; \xi)} \diff x} - \frac{\int_{L_{\alpha}(\xi)} \nu_h(x | \xi) \exp\prn[\big]{-\lambda (H-h) f(x; \xi)} \diff x}{\int_{\X(\xi)} \nu_h(x | \xi) \exp\prn[\big]{-\lambda (H-h) f(x; \xi)} \diff x}} \\
    \leq{}& \abs*{\frac{\int_{L_{\alpha}(\xi)} \pi_h(x | \xi) \exp\prn[\big]{-\lambda (H-h) f(x; \xi)} \diff x}{\int_{\X(\xi)} \pi_h(x | \xi) \exp\prn[\big]{-\lambda (H-h) f(x; \xi)} \diff x} - \frac{\int_{L_{\alpha}(\xi)} \pi_h(x | \xi) \exp\prn[\big]{-\lambda (H-h) f(x; \xi)} \diff x}{\int_{\X(\xi)} \nu_h(x | \xi) \exp\prn[\big]{-\lambda (H-h) f(x; \xi)} \diff x}} \\
    &\hspace{1em} {}+ \abs*{\frac{\int_{L_{\alpha}(\xi)} \pi_h(x | \xi) \exp\prn[\big]{-\lambda (H-h) f(x; \xi)} \diff x}{\int_{\X(\xi)} \nu_h(x | \xi) \exp\prn[\big]{-\lambda (H-h) f(x; \xi)} \diff x} - \frac{\int_{L_{\alpha}(\xi)} \nu_h(x | \xi) \exp\prn[\big]{-\lambda (H-h) f(x; \xi)} \diff x}{\int_{\X(\xi)} \nu_h(x | \xi) \exp\prn[\big]{-\lambda (H-h) f(x; \xi)} \diff x}} \\
    ={}& \frac{ \abs[\big]{\int_{L_{\alpha}(\xi)} \pi_h(x | \xi) \exp\prn[\big]{-\lambda (H-h) f(x; \xi)} \diff x} }{ \abs[\big]{\int_{\X(\xi)} \pi_h(x | \xi) \exp\prn[\big]{-\lambda (H-h) f(x; \xi)} \diff x} } \cdot \frac{ \abs[\big]{\int_{\X(\xi)} \prn[\big]{ \pi_h(x | \xi) - \nu_h(x | \xi) } \exp\prn[\big]{-\lambda (H-h) f(x; \xi)} \diff x} }{ \abs*{\int_{\X(\xi)} \nu_h(x | \xi) \exp\prn[\big]{-\lambda (H-h) f(x; \xi)} \diff x} } \\
    &\hspace{1em} {}+ \frac{ \abs[\big]{\int_{L_{\alpha}(\xi)} \prn[\big]{ \pi_h(x | \xi) - \nu_h(x | \xi) } \exp\prn[\big]{-\lambda (H-h) f(x; \xi)} \diff x} }{ \abs[\big]{\int_{\X(\xi)} \nu_h(x | \xi) \exp\prn[\big]{-\lambda (H-h) f(x; \xi)} \diff x} } \\
    \leq{}& \frac{ \abs[\big]{\int_{\X(\xi)} \prn[\big]{ \pi_h(x | \xi) - \nu_h(x | \xi) } \exp\prn[\big]{-\lambda (H-h) f(x; \xi)} \diff x} }{ \abs*{\int_{\X(\xi)} \nu_h(x | \xi) \exp\prn[\big]{-\lambda (H-h) f(x; \xi)} \diff x} } + \frac{ \abs[\big]{\int_{L_{\alpha}(\xi)} \prn[\big]{ \pi_h(x | \xi) - \nu_h(x | \xi) } \exp\prn[\big]{-\lambda (H-h) f(x; \xi)} \diff x} }{ \abs[\big]{\int_{\X(\xi)} \nu_h(x | \xi) \exp\prn[\big]{-\lambda (H-h) f(x; \xi)} \diff x} }.
  \end{align*}
  \vspace{-4mm}
  \end{strip}
  
  \noindent Now by definition of $\dTV(\cdot, \cdot)$, we have
  \begin{align}
    \abs[\big]{\pi_h(S | \xi) - \nu_h(S | \xi)}
    &\leq \abs*{\int_{S} \prn[\big]{\pi_h(x | \xi) - \nu_h(x | \xi)} \diff x} \label{eq:3-2-sampling_iterative-dTV}\\
    &\leq \int_{S} \abs[\big]{\pi_h(x | \xi) - \nu_h(x | \xi)} \diff x 
    \leq \eps_h \nonumber
  \end{align}
  for any measurable subset $S$ of $\X(\xi)$. Therefore, for any $S \in \set{L_{\alpha}(\xi), \X(\xi)}$,
  \begin{align*}
    &\abs*{\int_{S} \prn[\big]{ \pi_h(x | \xi) - \nu_h(x | \xi) } \exp\prn[\big]{ -\lambda (H-h) f(x; \xi) } \diff x} \\
    \leq{}& \exp\prn[\big]{ -\lambda(H-h) \fs(\xi) } \cdot \abs*{\int_{S} \prn[\big]{ \pi_h(x | \xi) - \nu_h(x | \xi) } \diff x} \\
    \leq{}& \exp\prn[\big]{ -\lambda(H-h) \fs(\xi) } \cdot \eps_h.
  \end{align*}
  Meanwhile, we also have
  \begin{align*}
    &\abs*{\int_{\X(\xi)} \nu_h(x | \xi) \exp\prn[\big]{ -\lambda (H-h) f(x; \xi) } \diff x} \\
    \geq{}& \abs*{\int_{L_{\alpha}(\xi)} \nu_h(x | \xi) \exp\prn[\big]{ -\lambda (H-h) f(x; \xi) } \diff x} \\
    \geq{}& \exp\prn[\big]{ -\lambda(H-h) (\fs(\xi)+\alpha) } \cdot \abs*{\int_{L_{\alpha}(\xi)} \nu_h(x | \xi) \diff x} \\
    ={}& \exp\prn[\big]{ -\lambda(H-h) (\fs(\xi)+\alpha) } \cdot \nu_h\prn[\big]{ L_{\alpha}(\xi) \big| \xi}.
  \end{align*}
  Plug the above two inequalities back, and we conclude that
  \begin{align*}
    &\abs*{ \rho_h\prn[\big]{ L_{\alpha}(\xi) \big| \xi} - \rho_{h-1}\prn[\big]{ L_{\alpha}(\xi) \big| \xi} } \\
    \leq{}& \frac{2 \cdot \exp\prn[\big]{-\lambda(H-h) \fs(\xi)} \cdot \eps_h}{\exp\prn[\big]{-\lambda(H-h) (\fs(\xi)+\alpha)} \cdot \nu_h\prn[\big]{L_{\alpha}(\xi) \big| \xi}} \\
    ={}& \frac{2 \exp(\lambda\alpha (H-h))}{\nu_h\prn[\big]{ L_{\alpha}(\xi) \big| \xi}} \cdot \eps_h.
  \end{align*}
  This completes the proof.
\end{proof}

We proceed to Step 2 by showing the following lower bound on $\nu_h\prn[\big]{ L_{\alpha}(\xi) \big| \xi}$.

\begin{lemma}\label{thm:3-2-proof_lemmas-nu_bound}
  Under the assumptions of \Cref{thm:3-2-sampling_policy_iterative}, we have
  \begin{equation*}
    \nu_h\prn[\big]{ L_{\alpha}(\xi) \big| \xi}
    \geq \pi_0\prn[\big]{ L_{\alpha}(\xi) \big| \xi} - \sum_{\ell=1}^{h-1} \eps_{\ell},~ \forall \xi \in \varXi.
  \end{equation*}
\end{lemma}

\begin{proof}
  Note that, by definition of the distribution $\nu_h$, we have
  \begin{align*}
    \nu_h\prn[\big]{ L_{\alpha}(\xi) \big| \xi} &= \int_{L_{\alpha}(\xi)} \nu_h(x | \xi) \diff x \\
    &= \frac{\int_{L_{\alpha}(\xi)} \pi_{h-1}(x | \xi) \exp\prn[\big]{ -\lambda f(x; \xi) } \diff x}{\int_{\X(\xi)} \pi_{h-1}(x | \xi) \exp\prn[\big]{ -\lambda f(x; \xi) } \diff x} \\
    &= \frac{1}{ 1 + \frac{\int_{\X(\xi) \setminus L_{\alpha}(\xi)} \pi_{h-1}(x | \xi) \exp(-\lambda f(x; \xi)) \diff x}{\int_{L_{\alpha}(\xi)} \pi_{h-1}(x | \xi) \exp(-\lambda f(x; \xi)) \diff x} } \\
    &\geq \frac{1}{ 1 + \frac{\exp(-\lambda (\fs(\xi)+\alpha)) \cdot \pi_{h-1}(\X(\xi) \setminus L_{\alpha}(\xi) | \xi)}{\exp(-\lambda (\fs(\xi)+\alpha)) \cdot \pi_{h-1}(L_{\alpha}(\xi) | \xi)} } \\
    &= \frac{\pi_{h-1}\prn[\big]{ L_{\alpha}(\xi) \big| \xi}}{\pi_{h-1}\prn[\big]{L_{\alpha}(\xi) \big| \xi} + \pi_{h-1}\prn[\big]{ \X(\xi) \setminus L_{\alpha}(\xi) | \xi}} \\
    &= \pi_{h-1}\prn[\big]{ L_{\alpha}(\xi) \big| \xi} \\
    &\geq \nu_{h-1}\prn[\big]{ L_{\alpha}(\xi) \big| \xi } - \eps_{h-1},
  \end{align*}
  where the first inequality is by definition of the level set $L_{\alpha}(\xi)$, and the second inequality is an application of \eqref{eq:3-2-sampling_iterative-dTV}.
  Since $\nu_{h-1}$ appears on the right-hand side, we shall iteratively apply the above inequality to obtain the lower bound
  \begin{equation*}
    \nu_h\prn[\big]{ L_{\alpha}(\xi) \big| \xi}
    \geq \pi_0\prn[\big]{ L_{\alpha}(\xi) \big| \xi} - \sum_{\ell=1}^{h-1} \eps_{\ell}.
  \end{equation*}
  This completes the proof.
\end{proof}

Now we are ready to finish the proof by collecting the bounds shown above.

\begin{proof}[Proof of \Cref{thm:3-2-sampling_policy_iterative}]
  To apply \Cref{thm:3-2-proof_lemmas-stepwise_error}, note that by telescoping we have
  \begin{align*}
    &\pi_H\prn[\big]{ L_{\alpha}(\xi) \big| \xi } \\
    ={}& \rho_0\prn[\big]{ L_{\alpha}(\xi) \big| \xi } + \sum_{h=1}^{H} \prn[\Big]{ \rho_h\prn[\big]{ L_{\alpha}(\xi) \big| \xi } - \rho_{h-1}\prn[\big]{ L_{\alpha}(\xi) \big| \xi } } \\
    \geq{}& \rho_0\prn[\big]{ L_{\alpha}(\xi) \big| \xi } - \sum_{h=1}^{H} \abs[\Big]{ \rho_h\prn[\big]{ L_{\alpha}(\xi) \big| \xi } - \rho_{h-1}\prn[\big]{ L_{\alpha}(\xi) \big| \xi } } \\
    \geq{}& \rho_0\prn[\big]{ L_{\alpha}(\xi) \big| \xi } - \sum_{h=1}^{H} \frac{2 \exp\prn[\big]{ \lambda\alpha (H-h) }}{\nu_h\prn[\big]{ L_{\alpha}(\xi) \big| \xi}} \cdot \eps_h.
  \end{align*}
  To further lower bound $\nu_h\prn[\big]{ L_{\alpha}(\xi) \big| \xi}$, we shall choose $\set{\eps_h}$ (specified later) such that
  \begin{equation}\label{eq:3-2-sampling_iterative-sum_eps}
    \sum_{\ell=1}^{H} \eps_{\ell} < \frac{C_0}{2} \cdot \alpha^{d_x},
  \end{equation}
  and hence, by \Cref{thm:2-2-level_set_lemma},
  \begin{equation*}
    \sum_{\ell=1}^{h-1} \eps_{\ell}
    \leq \sum_{\ell=1}^{H} \eps_{\ell}
    < \frac{C_0}{2} \cdot \alpha^{d_x}
    \leq \frac{1}{2} \pi_0\prn[\big]{ L_{\alpha}(\xi) \big| \xi }.
  \end{equation*}
  Consequently, we shall apply \Cref{thm:3-2-proof_lemmas-nu_bound} to obtain
  \begin{equation*}
    \nu_h\prn[\big]{ L_{\alpha}(\xi) \big| \xi }
    \geq \frac{1}{2} \pi_0\prn[\big]{ L_{\alpha}(\xi) \big| \xi }
    > \frac{C_0}{2} \cdot \alpha^{d_x},
  \end{equation*}
  for any $h \in [H]$ and $\xi \in \varXi$. Plugging the above inequality into the telescoping bound, we have
  \begin{alignfit}
    \pi_H\prn[\big]{ L_{\alpha}(\xi) \big| \xi }
    > \pistar_H\prn[\big]{ L_{\alpha}(\xi) \big| \xi } - \frac{4}{C_0 \cdot \alpha^{\dx}} \sum_{h=1}^{H} \exp\prn[\big]{ \lambda\alpha (H-h) } \eps_h,
  \end{alignfit}
  where we use the fact that $\rho_0 = \pistar_H$. We proceed by bounding the two terms separately. For the first term, note that by \Cref{thm:3-1-sampling_policy_ideal} (where we take $\varLambda = \lambda H$), when we choose sufficiently large $H \in \Z_+$ such that 
  \begin{equation}\label{eq:3-2-sampling_iterative-T_bound}
    H > \frac{2}{\lambda \alpha} \prn*{ \dx \log\prn*{\frac{2}{\alpha}} + \log\prn*{\frac{2}{C_0 \delta}} },
  \end{equation}
  we can guarantee that
  \begin{equation*}
    \pistar_H\prn[\big]{ L_{\alpha}(\xi) \big| \xi } > 1 - \delta/2.
  \end{equation*}
  For the second term, when we choose $\set{\eps_h}$ to be sufficiently small, such that
  \begin{equation}\label{eq:3-2-sampling_iterative-eps_bound}
    \eps_h < \frac{C_0 \cdot \alpha^{\dx}}{8H} \exp\prn[\big]{ -\lambda\alpha (H-h) } \delta,~ \forall h \in [H],
  \end{equation}
  we can guarantee that
  \begin{equation*}
    \frac{4}{C_0 \cdot \alpha^{\dx}} \sum_{h=1}^{H} \exp(\lambda\alpha (H-h)) \eps_h
    \leq \delta / 2.
  \end{equation*}
  Recall that we also need to ensure \eqref{eq:3-2-sampling_iterative-sum_eps}, which requires $\delta$ to be sufficiently small. In fact, we may set
  \begin{equation}\label{eq:3-2-sampling_iterative-delta_bound}
    \delta < 4H\prn[\big]{ 1 - \exp(-\lambda\alpha) }
  \end{equation}
  to guarantee that
  \begin{align*}
    \sum_{\ell=1}^{H} \eps_{\ell}
    &< \frac{\delta C_0 \cdot \alpha^{\dx}}{8H} \sum_{h=0}^{H-1} \exp(-\lambda\alpha h) \\
    &< \frac{C_0 \cdot \alpha^{\dx}}{8H \prn[\big]{1-\exp(-\lambda\alpha)}} \cdot \delta
    < \frac{C_0}{2} \cdot \alpha^{\dx}.
  \end{align*}
  As a conclusion, when we take $T$, $\eps_h$ and $\delta$ as follows (see \eqref{eq:3-2-sampling_iterative-T_bound}, \eqref{eq:3-2-sampling_iterative-eps_bound} and \eqref{eq:3-2-sampling_iterative-delta_bound}):
  \begin{align*}
    \delta &< 4H\prn[\big]{ 1 - \exp(-\lambda\alpha) },\\
    H &> \frac{2}{\lambda \alpha} \prn*{ \dx \log\prn*{\frac{2}{\alpha}} + \log\prn*{\frac{2}{C_0 \delta}} } = O\prn*{\frac{\dx}{\alpha} \log\prn*{ \frac{1}{\alpha \delta}}},\\
    \eps_h &< \frac{C_0 \cdot \alpha^{\dx}}{8H} \exp\prn[\big]{ -\lambda\alpha (H-h) } \delta,~ \forall h \in [H],
  \end{align*}
  we can show that $\pi_H\prn[\big]{ L_{\alpha}(\xi) \big| \xi } > 1-\delta$, $\forall \xi \in \varXi$.
\end{proof}

\begin{remark}
  The current bound is pessimistic in the choice of $\set{\eps_h}$, in the sense that we assume the smallest possible $\alpha$-level sets under the assumptions. When $\pi_0\prn[\big]{ L_{\alpha}(\xi) \big| \xi}$ is significantly bigger than $C_0 \cdot \alpha^{\dx}$ (e.g., when there are multiple minimizers), we can obtain a much looser bound on $\eps_h$.
\end{remark}

\subsection{Sample Complexity of D-IMPL}\label{sec:apdx-proof_D_IMPL}
In this section, we present the proof of \Cref{thm:main_complexity-D_IMPL}. For clarity, we first supplement some technical details regarding the implementation of DDPMs and the learning setting, and then proceed to the proof outline. Detailed proofs will be shown in accordance with the outline.

\subsubsection{Technical Preliminaries}\label{sec:apdx-proof_D_IMPL_technical}

We first supplement some necessary technical details regarding DDPM and the notations used in the analysis.

\textbf{Ornstein-Unlenbeck Noise Schedule.} For the sake of simplified theoretical analysis, we consider the denoising schedule $\bar{\alpha}_k = \exp(-2k)$, i.e.,
\begin{equation*}
  x_k = \exp(-k) x_0 + \sqrt{1 - \exp(-2k)} z_k,~ \text{where}~z_k \sim \N(0, I),
\end{equation*}
which is a discretized version of the following \emph{Ornstein-Unlenbeck process}: \cite{oksendal2003stochastic}
\begin{equation*}
  \diff x_t = - x_t \diff t + \sqrt{2} \diff B_t.
\end{equation*}
It is known in literature that that SDEs can always be reversed \cite{anderson1982reverse}, and the reversed Ornstein-Unlenbeck process $\set{\bar{x}_t}$ follows the following SDE:
\begin{equation*}
  \diff \bar{x}_t = -\prn[\big]{ \bar{x}_t + 2\nabla \log q_t(\bar{x}_t) } \diff t + \sqrt{2} \diff \bar{B}_t,
\end{equation*}
where $\bar{B}_t$ denotes the reverse-time Brownian motion. It is clear that the reversed SDE motivates the backward process sampling process, where we start from $x_T \sim \N(\bm{0}, I)$ for a sufficiently large $T$ (so that the distribution at time $T$ is sufficiently close to $\N(\bm{0}, I)$), and then simulate the backward process using discretized integration method at $K$ discretization points $0 < t_0 < t_1 < \cdots < t_K = T$. Here we adopt the ``early stopping'' trick to take a sufficiently small, yet non-zero, starting time $t_0$, which has proven to show significantly improvement in terms of inference performance.

\textbf{Learning Setting.} In the analysis, we mainly work with the reweighted loss and its empirical version, namely
\begin{alignfit}
  \L_{h,k}(\theta) &= \E[\xi, \omega, x_h \sim \pi_h(\cdot | \xi)]{Z(\xi) \cdot \norm[\big]{s^{\theta}(x_h; \xi) - \ssc_{h,k}(x_h; \xi)}^2}, \\
  \hL_{h,k}(\theta) &= \dfrac{1}{N_h} \sum\limits_{i=1}^{N_h} \exp\prn[\big]{ -\lambda f(x_i; \xi) } \cdot \norm*{s^{\theta}(x_i; \xi_i) - (-\omega_i / \sigma_k)}^2,
\end{alignfit}
where the expectation is with respect to $\xi \sim \rho$, $\omega \sim \N(\bm{0}, I)$; $\ssc_{h,k}(x; \xi) = -\omega / \sigma_k$, and the weights $Z(\xi)$ are defined in \eqref{eq:4-2-reweighted_objective_normalizing}; the dataset $\D_h = \set{(x_i, \xi_i, \omega_i) \mid i \in [N_h]}$ is sampled accordingly. For the sake of analysis, we also define the normalized loss $\tL_{h,k}(\theta) := \L_{h,k}(\theta) / Z_{\min}$, i.e.,
\begin{alignfit}
  \tL_{h,k}(\theta) = \E[\xi, \omega, x_h \sim \pi_h(\cdot | \xi)]{g(\xi) \cdot \norm[\big]{s^{\theta}(x_h; \xi) - \ssc_{h,k}(x_h; \xi)}^2},
\end{alignfit}
where the normalized weights $g(\xi) := Z(\xi) / Z_{\min}$, such that $g(\xi) \in [1, \gamma]$ as shown in \Cref{thm:4-1-lemma_weights}. 
For notational simplicity, omit the subscript $h$ as we consider a specific $h$\tsup{th} iteration, and write $\hstheta_k$ and $\tstheta_k$ for the minimizers of $\hL_k(\theta)$ and $\tL_k(\theta)$ (or equivalently $\L_k(\theta)$), respectively; i.e.,
\begin{equation*}
  \tstheta_k := \arg\min_{\theta \in \varTheta} \tL_k(\theta),\quad
  \hstheta_k := \arg\min_{\theta \in \varTheta} \hL_k(\theta).
\end{equation*}
We also use the shorthand notations $\hssc_k := s^{\hstheta_k}$ and $\tssc_k := s^{\tstheta_k}$, and overload the notations $\tL_k(s^{\theta}) := \tL_k(\theta)$ and $\hL_k(s^{\theta}) := \hL_k(\theta)$ for clarity. Further, let $\hq_k$ denote the distribution recovered by DDPM using the learned score function $\hssc_k$.

\medskip
\textbf{Proof Structure.} For clarity, the proof is organized in the following three steps, which is consistent with the proof sketch in the main text:
\begin{itemize}
  \item \textbf{Step 1}: We first apply a known result in literature relating two key quantities in the proof: the total variation distance between $\hq_0(\cdot | \xi)$ (the distribution recovered by diffusion process) and $q_0(\cdot | \xi)$ (the target distribution), and the mean square error $E_k(\xi) := \E[x \sim q_k(\cdot | \xi)]{\norm{\hsc_k(x; \xi) - \ssc_k(x; \xi)}^2}$.
  
  \item \textbf{Step 2}: We further relate $E_k(\xi)$ to the loss functions, and derive a PAC-learning-style upper bound for the loss.
  
  \item \textbf{Step 3}: Finally, we conclude the proof of the theorem and a corollary relating back to the performance of D-IMPL.
\end{itemize}

\subsubsection{Relate Total Variation Distance to Mean Square Error}

By directly applying a result known in literature (Theorem 2 in \cite{gaur2025DDPM}), we can effectively bound $\dTV\prn[\big]{ q_0(\cdot | \xi), \hq_0(\cdot | \xi) }$ by $E_k(\xi)$, as summarized in the following lemma.

\begin{lemma}\label{thm:4-D-proof-lemma_dTV}
  Under the premises of \Cref{thm:main_complexity-D_IMPL}, for any $\xi \in \varXi$, let $\hq_0(\cdot | \xi)$ be the distribution recovered by DDPM using learned score functions $\set{\hsc_k}_{k=1}^{K}$, and we have
  \begin{equation*}
    \dTV\prn[\big]{ q_0(\cdot | \xi), \hq_0(\cdot | \xi) } \leq C + \frac{1}{2} \sqrt{\sum_{k=0}^{K} E_k(\xi) \cdot (t_{k+1} - t_k)},
  \end{equation*}
  where $C$ is a constant determined by diffusion settings $t_0$, $T$, $K$ and $\eps_{\F}$ as
  \begin{equation*}
    C = O\prn[\big]{\sqrt{t_0} \log(1/t_0)} + O\prn[\big]{\exp(-T)} + O\prn[\big]{1/\sqrt{K}} + \eps_{\F}.
  \end{equation*}
\end{lemma}

\begin{proof}
  To apply Theorem 2 in \cite{gaur2025DDPM} (see eq. (13) there), it only suffices to verify Assumption 5 there (Assumption 3 is the same as item (1) in \Cref{assum:4-D-diffusion-hypothesis_class}; all the other assumptions are not used in the proof of this intermediate result where the optimization error is $0$), which requires the data distribution to be sub-Gaussian. Indeed, all distributions with bounded supports are sub-Gaussian by nature. More specifically, one can easily verify that, under the boundedness assumption in \Cref{assum:2-1-minimizers},
  \begin{equation*}
    \Prob[x \sim q_0(\cdot | \xi)]{|x| \geq t}
    \left\lbrace \begin{array}{@{}l@{}}
      \leq 1 ~ (\forall t \leq C_x) \\
        =  0 ~ (\forall t > C_x)
    \end{array} \right\rbrace \leq 2 \exp\prn*{ -\frac{t^2}{(C_x/2)^2} }.
  \end{equation*}
  Hence the proof is completed.
\end{proof}

\subsubsection{Relate Mean Square Error to Score-Matching Loss}

From the above lemma, it only suffices to bound the mean square error $E_k(\xi)$. For this purpose, note that we have the following error decomposition
\begin{align*}
  E_k(\xi) 
  &= \E[x \sim q_k(\cdot | \xi)]{\norm{\hssc_k(x; \xi) - \ssc_k(x; \xi)}^2} \\
  &\leq 2\E[x \sim q_k(\cdot | \xi)]{\norm{\hssc_k(x; \xi) - \tssc_k(x; \xi)}^2} \\
  &\hspace{4em} {}+ 2\E[x \sim q_k(\cdot | \xi)]{g(\xi) \cdot \norm{\tssc_k(x; \xi) - \ssc_k(x; \xi)}^2} \\
  &\leq 2\E[x \sim q_k(\cdot | \xi)]{\norm{\hssc_k(x; \xi) - \tssc_k(x; \xi)}^2} + 2 \eps_{\F},
\end{align*}
where we use $g(\xi) \geq 1$ and $\norm{a+b}^2 \leq 2\norm{a}^2 + 2\norm{b}^2$ in the first inequality, and we plug in \Cref{assum:4-D-diffusion-hypothesis_class} in the second inequality. To facilitate the proof, we first show the following characterization of the reweighted objective.

\begin{lemma}[Reweighted objective]\label{thm:4-1-lemma_weights}
  Consider the following reweighted D-IMPL objective
  \begin{equation*}
    \tL(\theta) := \E[\xi \sim \rho, x\sim \pi_h(\cdot|\xi)]{Z(\xi) \cdot \ell_{h,k}(\theta)},
  \end{equation*}
  where $Z(\xi)$ is the (intractable) normalization factor of the $h$\tsup{th} iterative update, i.e.,
  \begin{equation}\label{eq:4-2-reweighted_objective_normalizing}
    Z(\xi) := \int_{\X(\xi)} \pi_{h-1}(x | \xi) \exp\prn[\big]{ -\lambda f(x; \xi) } \diff x.
  \end{equation}
  The reweighted objective $\tL(\theta)$ is tractable in that it can be equivalently rewritten as
  \begin{equation}
    \tL(\theta) = \E[\xi \sim \rho, x \sim \pi_{h-1}(\cdot | \xi)]{\exp\prn[\big]{ -\lambda f(x; \xi) } \cdot \ell_{h,k}(\theta)}.
  \end{equation}
  Further, assuming $\pi_h\prn[\big]{ L_{\alpha}(\xi) \big| \xi} > 1 - \delta$, $\forall \xi \in \varXi$ for some $\alpha > 0$ and $\delta \in (0,1)$, there exist positive constants $Z_{\min} < Z_{\max}$ independent of $\xi$, such that $\frac{Z_{\max}}{Z_{\min}} \leq \frac{\exp(\lambda\alpha)}{1-\delta} =: \gamma$ is constantly upper bounded (with $\gamma$ independent of $\xi$), and
  \begin{equation*}
    Z_{\min} \leq Z(\xi) \leq Z_{\max},\quad \forall \xi \in \varXi.
  \end{equation*}
\end{lemma}

\begin{proof}
  Note that our choice of $Z(\xi)$ ensures the following relation to hold:
  \begin{equation*}
    \pi_h(x | \xi) = \frac{\pi_{h-1}(x | \xi) \exp\prn[\big]{ -\lambda f(x;\xi) } }{ Z(\xi) }.
  \end{equation*}
  Therefore, we can directly write
  \begin{align*}
    \tL(\theta) &= \E[\xi \sim \rho, x \sim \pi_{h-1}(\cdot | \xi)]{Z(\xi) \cdot \frac{\pi_h(x | \xi)}{\pi_{h-1}(x | \xi)} \cdot \ell_{h,k}(\theta)} \\
    &= \E[\xi \sim \rho, x \sim \pi_{h-1}(\cdot | \xi)]{\exp\prn[\big]{ -\lambda f(x; \xi) } \cdot \ell_{h,k}(\theta)}, 
  \end{align*}
  which now appears in a tractable form. To show the boundedness of $Z(\xi)$, note that by assumption we have
  \begin{align*}
    Z(\xi) &\geq \exp\prn[\big]{ -\lambda(f^*(x;\xi) + \alpha) } \cdot \int_{L_{\alpha}(\xi)} \pi(x; \xi) \diff x \\
    &\geq (1-\delta) \exp\prn[\big]{ -\lambda(f^*(x;\xi) + \alpha) }
  \end{align*}
  lower bounded. On the other hand, it is evident that
  \begin{align*}
    Z(\xi) &\leq \exp\prn[\big]{-\lambda f^*(x;\xi)} \cdot \int_{\X(\xi)} \pi(x; \xi) \diff x \\
    &= \exp\prn[\big]{-\lambda f^*(x;\xi)}.
  \end{align*}
  Therefore, we shall take
  \begin{align*}
    Z_{\min} &= (1-\delta) \exp\prn[\big]{ -\lambda(f^*(x;\xi) + \alpha) },\\
    Z_{\max} &= \exp\prn[\big]{-\lambda f^*(x;\xi)},
  \end{align*}
  such that $\frac{Z_{\max}}{Z_{\min}} \leq \frac{\exp(\lambda\alpha)}{1-\delta} =: \gamma$ is constantly bounded.
\end{proof}

\begin{remark}
  We point out that, if the hypothesis family $\F$ is \emph{realizable} (i.e. $\stheta \in \F$), then $\stheta$ is still the minimizer of $\tL(\theta)$ after reweighting, which, along with the boundedness of $Z(\xi)$, justifies the choice of the reweighted objective.
\end{remark}

Now we proceed to bound the term $\norm{\hssc_k(x; \xi) - \tssc_k(x; \xi)}^2$ in the expectation, relying on the following lemma.

\begin{lemma}\label{thm:4-D-proof-lemma_Exi}
  Under the premises of Theorem \ref{thm:main_complexity-D_IMPL}, for any $\xi \in \varXi$, any $\delta \in (0,1)$ and any $\eps > 0$, when we take a sufficiently large dataset such that
  \begin{equation*}
    N_h > \frac{288 \mu^2 \gamma^2 K_{\varTheta}^2 C_{\score}^2 \log(4/\delta)}{\eps^2} = O\prn*{\frac{\log(1/\delta)}{\eps^2}},
  \end{equation*}
  where $C_{\score} := 2 C_{\F}^2 + 2 \prn[\big]{\frac{2C_x}{1 - \exp(-2t_0)}}^2$, we can guarantee
  \begin{equation*}
    E_k(\xi) < \eps + 2\eps_{\F}
  \end{equation*}
  with probability at least $1-\delta$.
\end{lemma}

\begin{proof}
  Based on the relationship above, We first relate $\norm{\hssc_k(x; \xi) - \tssc_k(x; \xi)}^2$ to the difference in losses. Indeed, by Assumptions \ref{assum:4-D-diffusion-hypothesis_class} and \ref{assum:4-D-diffusion-PL_loss}, we have
  \begin{align*}
    \norm[\big]{\hssc_k(x; \xi) - \tssc_k(x; \xi)}^2
    &\leq K_{\varTheta} \norm[\big]{\hstheta_k - \tstheta_k}^2 \\
    &\leq \mu K_{\varTheta} \prn[\big]{\tL_k(\hssc_k) - \tL_k(\tssc_k)},
  \end{align*}
  where we use the quadratic growth property of PL functions. Therefore, it only suffices to bound $\tL_k(\hssc_k) - \tL_k(\tssc_k)$, a statistical error term. For this purpose, note that we have the canonical loss decomposition
  \begin{align*}
    \tL_k(\hssc_k) - \tL_k(\tssc_k)
    &\leq \prn[\big]{ \tL_k(\hssc_k) - \hL_k(\hssc_k) } + \prn[\big]{ \hL_k(\hssc_k) - \hL_k(\tssc_k) } \\
    &\hspace{4em} {}+ \prn[\big]{ \hL_k(\tssc_k) - \tL_k(\tssc_k) } \\
    &\leq \prn[\big]{ \tL_k(\hssc_k) - \hL_k(\hssc_k) } + \prn[\big]{ \hL_k(\tssc_k) - \tL_k(\tssc_k) } \\
    &\leq \sum_{s \in \S} \abs[\big]{ \tL_k(s) - \hL_k(s) },
  \end{align*}
  where $\S := \set{\hssc_k, \tssc_k}$. For the upper bound, we would like to apply the well-known bound involving Rademacher complexity (\Cref{thm:lemma-Rademacher_property}). For this purpose, we first have to show that the individual losses are uniformly bounded. Note that
  \begin{equation*}
     \norm[\big]{s(x; \xi) - \ssc_k(x; \xi)}^2
    \leq 2 \norm[\big]{s(x; \xi)}^2 + 2 \norm[\big]{\ssc_k(x; \xi)}^2,~ \forall s \in \S,
  \end{equation*}
  where we apply $\norm{a+b}^2 \leq 2\norm{a}^2 + 2\norm{b}^2$ again. Now the first term is directly bounded by \Cref{assum:4-D-diffusion-hypothesis_class}, while the second term can be directly computed as
  \begin{align*}
    \norm[\big]{\ssc_k(x; \xi)}
    &= \norm*{\frac{x_k - \exp(-t_k) x_0}{\sigma_k^2}}
    \leq \frac{\norm{x_k} + \norm{x_0}}{\sigma_k^2} \\
    &\leq \frac{2C_x}{1 - \exp(-2t_0)}
  \end{align*}
  by Assumption \ref{assum:2-1-minimizers}. As a conclusion, for any $s \in \S$,
  \begin{equation*}
    \norm[\big]{s(x; \xi) - \ssc_k(x; \xi)}^2
    \leq 2 C_{\F}^2 + 2 \prn*{\frac{2C_x}{1 - \exp(-2t_0)}}^2
    =: C_{\score},
  \end{equation*}
  and consequently $g(\xi) \cdot \norm{s(x; \xi) - \ssc_k(x; \xi)}^2 \leq \gamma C_{\score}$. Then by \Cref{thm:lemma-Rademacher_property} we have, with probability at least $1-\delta/2$,
  \begin{equation*}
    \abs[\big]{ \tL_k(s) - \hL_k(s) } \leq 2R(\S) + \gamma C_{\score} \sqrt{\frac{2 \log(4/\delta)}{N_k}},~ \forall s \in \S,
  \end{equation*}
  where $R(\S)$ is the Rademacher complexity of the finite hypothesis class $\S$. It is well-known that (see \Cref{thm:lemma-Rademacher_finite_class}), the Rademacher complexity of a bounded finite class $\S$ has a natural upper bound
  \begin{equation*}
    R(\S) \leq \gamma C_{\score} \sqrt{\frac{\log 4}{N_k}} < \gamma C_{\score} \sqrt{\frac{2 \log(4/\delta)}{N_k}}.
  \end{equation*}
  Therefore, with probability at least $1-\delta/2$, we have
  \begin{equation*}
    \abs[\big]{ \tL_k(s) - \hL_k(s) } \leq 3 \gamma C_{\score} \sqrt{\frac{2 \log(4/\delta)}{N_k}},~ \forall s \in \S,
  \end{equation*}
  and the overall bound now becomes
  \begin{align*}
    \norm[\big]{\hssc_k(x; \xi) - \tssc_k(x; \xi)}^2
    \leq 6 \mu \gamma K_{\varTheta} C_{\score} \sqrt{\frac{2 \log(4/\delta)}{N_k}},
  \end{align*}
  and consequently,
  \begin{equation*}
    E_k(\xi) \leq 12 \mu \gamma K_{\varTheta} C_{\score} \sqrt{\frac{2 \log(4/\delta)}{N_k}} + 2\eps_{\F}.
  \end{equation*}
  In other words, if we take
  \begin{equation*}
    N_h > \frac{288 \mu^2 \gamma^2 K_{\varTheta}^2 C_{\score}^2 \log(4/\delta)}{\eps^2},
  \end{equation*}
  using union bound we can guarantee that,
  \begin{equation*}
    E_k(\xi) < \eps + 2\eps_{\F}
  \end{equation*}
  with probability at least $1-\delta$. This completes the proof.
\end{proof}

\begin{remark}
  We point that here the product $\mu \gamma$ could be regarded as a constant. Specifically, if we scale up $\gamma$ by multiplying the loss by a constant, then $\mu$ will have to be shrunk by the same multiple. This is exactly what happens when we consider $\tL_k$ instead of $\L_k$ in the analysis.
\end{remark}

\subsubsection{The Main Theorem}

At this point, we are ready to prove the main theorem (\Cref{thm:main_complexity-D_IMPL}) regarding the sample complexity of the D-IMPL algorithm.

\begin{proof}[Proof of \Cref{thm:main_complexity-D_IMPL}]
  Combining \Cref{thm:4-D-proof-lemma_dTV} and \Cref{thm:4-D-proof-lemma_Exi}, we have, with probability at least $1-\delta$,
  \begin{align*}
    \dTV\prn[\big]{ q_0(\cdot | \xi), \hq_0(\cdot | \xi) } &\leq O\prn[\big]{\sqrt{t_0} \log(1/t_0)} + O\prn[\big]{1/\sqrt{K}} \\
    &{}+ O\prn[\big]{\exp(-T)} + \eps_{\F} + \frac{\sqrt{T}}{2} (\eps + 2\eps_{\F}).
  \end{align*}
  Now we shall take
  \begin{equation*}
    t_0 = O(\eps),\quad
    K = O\prn[\big]{ 1 / \eps^2 },\quad
    T = O\prn[\big]{ \log(1/\eps) },
  \end{equation*}
  to ensure
  \begin{equation*}
    \dTV\prn[\big]{ q_0(\cdot | \xi), \hq_0(\cdot | \xi) } \leq \widetilde{O}(\eps + \eps_{\F})
  \end{equation*}
  in this high-probability event. This completes the proof.
\end{proof}

\subsection{Technical Lemmas}

\begin{lemma}[Rademacher bound]\label{thm:lemma-Rademacher_property}
  Given a hypothesis class $\H = \set{h}$, a data distribution $x \sim p$, and a loss function $\ell(h, x)$ that is uniformly bounded by $\abs{\ell(h, x)} \leq c$, consider the expected and empirical losses $\L(h) = \E[x \sim p]{\ell(h, x)}$ and $\hL(h) = \tfrac{1}{N} \textstyle\sum_{i=1}^{N} \ell(h, x_i)$. Then, with probability at least $1-\delta$, the error bound
  \begin{equation*}
    \L(h) - \hL(h) \leq 2R_N(\H) + c \sqrt{\frac{\log(2/\delta)}{N}}
  \end{equation*}
  holds for all $h \in \H$, and specifically, for the ERM estimator $\hat{h}^{\star} = \min_{h \in \H} \hL(h)$, where $R(\H)$ denotes the \emph{empirical Rademacher complexity} of class $\H$. 
\end{lemma}

\begin{proof}
  See, e.g., Theorem 26.5 in \cite{shalev2014understanding}.
\end{proof}

\begin{lemma}[Rademacher complexity of a finite set]\label{thm:lemma-Rademacher_finite_class}
  For a finite class $\H = \set{h_k \mid k \in [m]}$ where all candidates are uniformly bounded by $b$ (i.e., $\norm{h_k}_{\infty} \leq b$), its empirical Rademacher complexity with respect to data $\set{x_i \mid i \in [N]}$ is bounded by
  \begin{equation*}
    R_N(\H) \leq b \sqrt{\frac{\log(2m)}{N}}.
  \end{equation*}
\end{lemma}

\begin{proof}
  This is a well-known corollary of Massart's Lemma (see, e.g., Lemma 26.8 in \cite{shalev2014understanding}).
\end{proof}

\subsection{Numerical Simulation (continued)}

In this section, we supplement some details, as well as additional results, of the numerical simulations.

\textbf{Setup.} We implement a 30-step DDPM with Ornstein-Unlenbeck denoising schedule. The score function is implemented by a 3-layer neural network with ReLU activation.

\textbf{Objectives.} To demonstrate the capability of the minimization policies learned by D-IMPL to capture complex, multi-modal minimizer sets, we test the D-IMPL algorithm in different objective functions, as summarized in \Cref{tab:5-functions-multi_modal}.

\vspace{-2mm}
\begin{table}[H]
    \centering
    \begin{tabular}{c|c|c}
      \specialrule{1.5pt}{0pt}{0pt}
      \textbf{Name} & \textbf{Function} & \textbf{Minimizers} \\
      \specialrule{0.4pt}{0pt}{0pt}
      \texttt{Himm} & $(x_1^2 + x_2 - \xi_1)^2 + (x_1 + x_2^2 - \xi_2)^2$ & various \\\hline
      \texttt{trig} & $\sin\prn[\big]{ (x_1 - \xi_1) \pi} + \cos\prn[\big]{ (x_2 - \xi_2) \pi}$ & countable \\\hline
      \texttt{quad-clp} & $\max\brac[\big]{(x_1-\xi_1)^2 + (x_2-\xi_2)^2, 0}$ & 1-d manifold \\\hline
      \texttt{quad-lin} & $(x_1 \sin \xi_1 + x_2 \cos \xi_2) ^2$ & 2-d manifold \\
      \specialrule{1.5pt}{0pt}{0pt}
    \end{tabular}
    \caption{Functions with multi-modal landscapes.}\label{tab:5-functions-multi_modal}
\end{table}
\vspace{-4mm}

\textbf{Convergence.} In addition, we showcase the convergence behavior of the D-IMPL algorithm with respect to \texttt{quad-lin} in \Cref{fig:5-convergence_curve}. It is evident that the test accuracy curve shows a clear pattern of ``stair cases'' that eventually approaches $100\%$, which is consistent with our iterative update scheme. 

\vspace{-2mm}
\begin{figure}[H]
  \centering
  \includegraphics[width=0.75\linewidth]{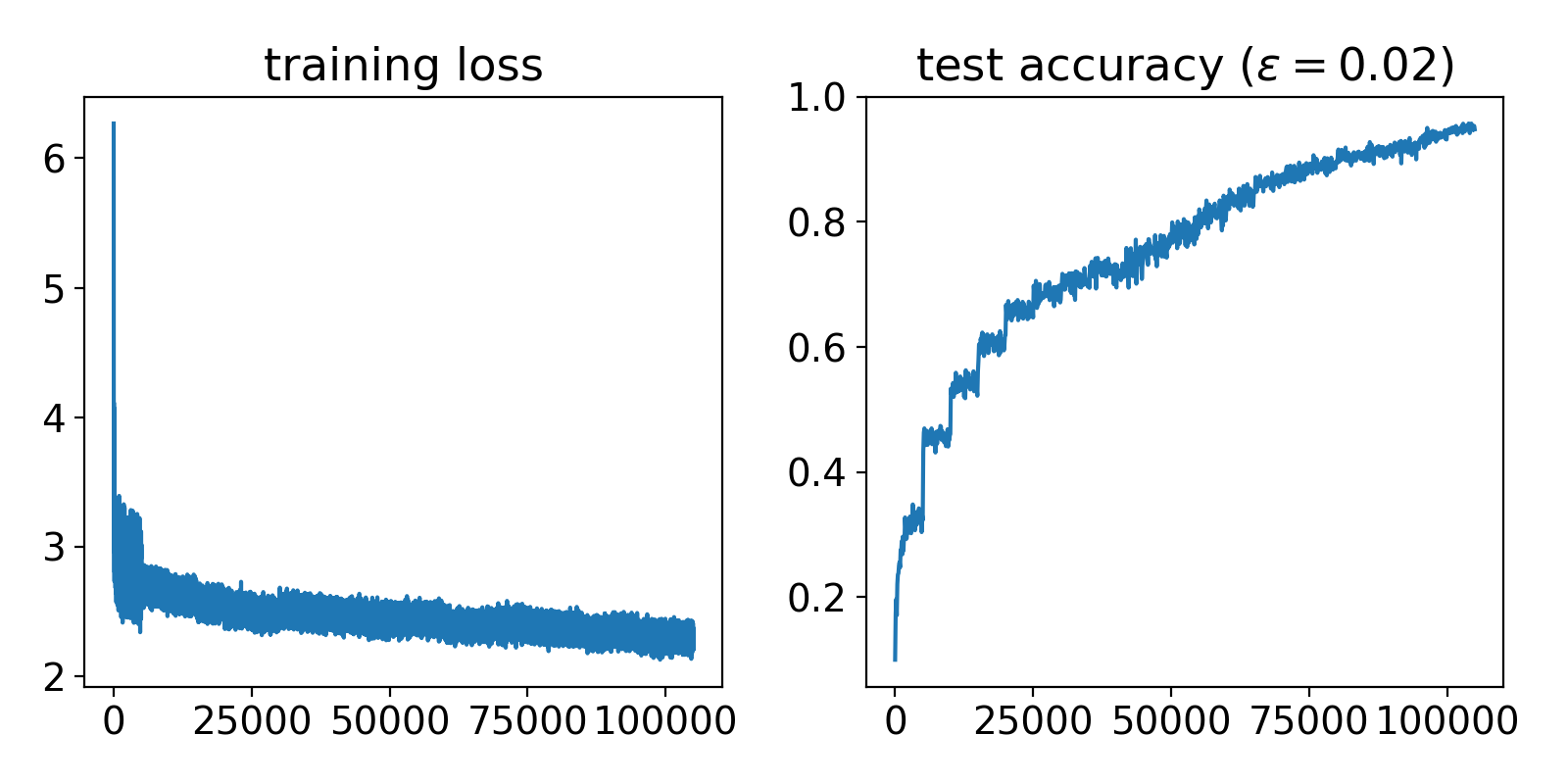}

  \vspace{-2mm}
  \caption{Convergence of D-IMPL for \texttt{quad-lin}.}\label{fig:5-convergence_curve}
\end{figure}

\vspace{-6mm}
\begin{figure}[H]
  \centering
  \includegraphics[width=\linewidth]{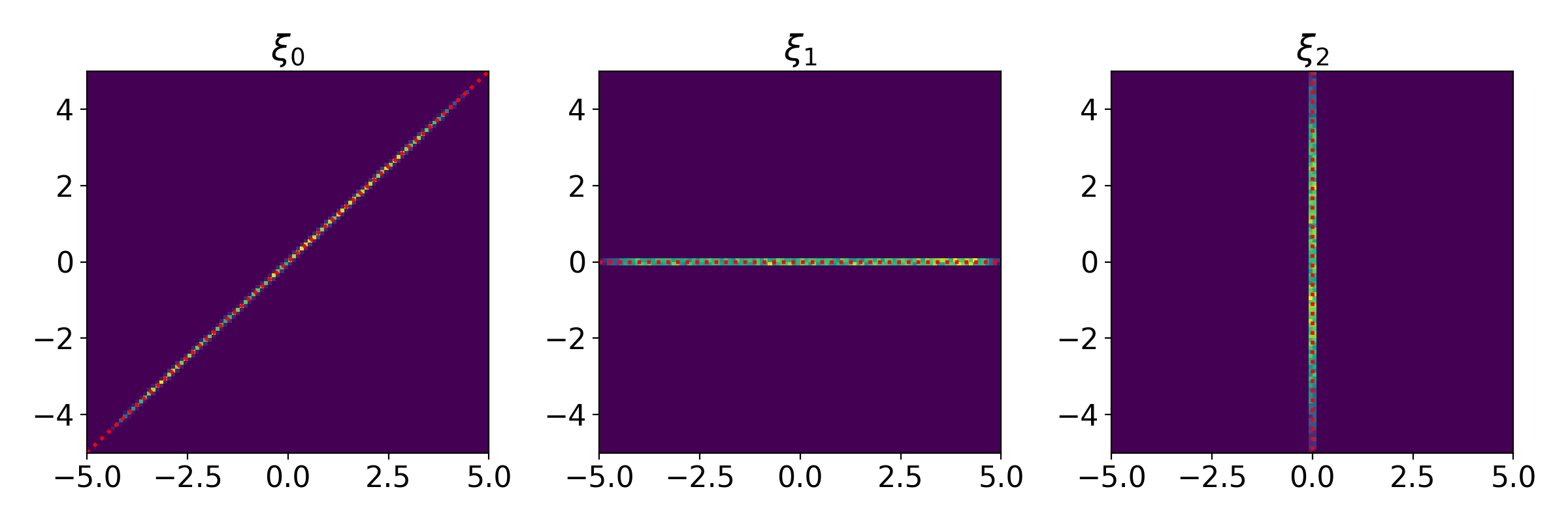}

  \vspace{-2mm}
  \caption{Learned policies for \texttt{quad-lin} (dotted line: ground truth; \textcolor{Purple}{color}\textcolor{Emerald}{-}\textcolor{Gold1}{scale}: learned).}
\end{figure}

\subsection{Benefits of Diffusion-Based Optimizers}

In this section, as an extension to \Cref{sec:settings-formulation}, we further discuss the benefits of these design objectives in details.

\textbf{Benefit 1: Requiring only black-box information.} It is assumed that, when learning the universal optimizer, we only have access to the objective via function value queries; in other words, the objective function is given as a black box, without exposing its internal structure. The call for black-box methods roots deeply in the limited level of access to objective functions in real-world settings. Indeed, gradient-based methods, as the most prominent and most widely-used methods nowadays, generally require direct access to the explicit form of the objective function, which may not be readily available due to the complexity of problems (for example, the objective function may not be known in closed forms, but rather, can only be approximated by simulation so that gradient computation is infeasible). The above challenge motivates the study of black-box optimization (BBO) methods, examples of which include Bayesian optimization methods \cite{wang2023recent, ren2025tsrsr}, zeroth-order optimization methods \cite{chen2022improve, ren2023escaping}, and other derivative-free search methods \cite{larson2019derivative}. Our setting inherits the idea of black-box objective access, and further generalizes it to a parameterized family of structurally similar instances.

\textbf{Benefit 2: Universal learning-based optimizer.} In our setting, instead of solving single optimization instances one at a time, we aim at learning a \emph{universal} optimizer that can simultaneously solve a parameterized family of optimizations. Here by ``simultaneously'' we mean the minimization policy produced at the end of the learning stage is able to solve all possible instances met in the inference stage. In real-world engineering practices, a wide range of tasks require solving a collection of similar optimization problems that are readily parameterized. To name just a few of them:
\begin{itemize}
  \item for resource allocation problems (e.g., optimal power flow problem discussed in \cite{hoseinpour2025diffopf}), we may view $\xi$ as the demand, $x$ as the supply, and $f(x; \xi)$ as the allocation cost;
  \item for receding horizon control problems, we may view $\xi$ as the current state, $x$ as the receding-horizon control sequence, and $f(x; \xi)$ as the control cost;
  \item for reinforcement learning (RL), we may view $\xi$ as the current state, $x$ as the current action, and $f(x; \xi)$ as the Q-function (i.e., cost-to-go function).
\end{itemize}
This is in stark contrast to classical BBO methods \cite{wang2023recent, ren2025tsrsr, chen2022improve, ren2023escaping, larson2019derivative}, which solves for a single minimizer on a per-instance basis. Therefore, even faced with a structurally similar optimization instance with slightly different parameters, the solver still needs to start optimizing from scratch, which could be infeasible since both the algorithm and the objective query could be expensive, time-consuming and/or capped by a limit. As a result, we are in desperate need of a general-purpose method to learn a universal optimizer for a parametric family of instances, such that the learned optimizer can solve a new problem without further querying the unseen objective function at the inference stage.

\textbf{Benefit 3: Capturing multi-modal landscapes.}  As discussed in the introduction, the proposed new solution concept of minimization policies automatically enables the resulting optimizer to capture complex optimization landscapes where the minimizer set may be multi-modal, continuous, or even forming a lower-dimensional manifold. This is remarkably different from the existing BBO methods that only output \emph{one} minimizer for each optimization problem instance, which may be insufficient for highly non-convex optimization landscapes with multiple or even a continuum of minimizers.
\end{document}